\documentclass[final,5p,times,twocolumn]{elsarticle}

\usepackage{amssymb,amsthm}
\usepackage{bm,mathtools}
\theoremstyle{remark}
\newtheorem{thm}{Theorem}
\newtheorem{prop}[thm]{Proposition}
\newtheorem{cor}[thm]{Corollary}

\newtheorem*{cor*}{Corollary}
\newtheorem*{ass*}{Assumption}
\theoremstyle{definition}
\newtheorem*{rem*}{Remark}
\newtheorem{lem}[thm]{Lemma}

\newtheoremstyle{named}{}{}{\itshape}{}{\bfseries}{.}{.5em}{\thmnote{#3}}
\theoremstyle{named}

\newcommand{\softmax}{\operatorname{softmax}}

\usepackage{color,xcolor,colortbl}
\definecolor{finesky}{HTML}{E6F5F0}
\definecolor{lightgray}{HTML}{ECECEC}
\definecolor{lightcyan}{rgb}{0.88,1,1}

\usepackage{caption}
\usepackage{subcaption}
\usepackage{graphicx}
\usepackage{booktabs}

\usepackage{multirow}
\usepackage{booktabs}
\usepackage{boldline,makecell}

\usepackage{hyperref}

\usepackage{soul}

\journal{Neurocomputing}

\begin{document}

\begin{frontmatter}



\title{Periocular Embedding Learning with Consistent Knowledge Distillation from Face}

\author[orgx]{Yoon Gyo Jung\corref{contrib}}
\ead{jung.yoo@northeastern.edu}
\author[orga]{Jaewoo Park\corref{contrib}}
\ead{park.jaewoo@aiv.ai}
\author[orgz]{Cheng Yaw Low}
\ead{chengyawlow@ibs.re.kr}
\author[orgv]{Jacky Chen Long Chai}
\ead{jackychai920@gmail.com}
\author[orgv]{Leslie Ching Ow Tiong}
\ead{leslie.tiong@samsung.com}
\author[orgy]{Andrew Beng Jin Teoh\corref{corauthor}}
\ead{bjteoh@yonsei.ac.kr}

\cortext[contrib]{Equal contribution}
\cortext[corauthor]{Corresponding author}

%
%

\address[orgx]{Electrical and Computer Engineering Department Northeastern University, Boston, United States of America}
\address[orga]{AiV Research, AiV Co., Seoul, South Korea}
\address[orgi]{Data Science Group, Institute for Basic Science, South Korea}
\address[orgs]{Mechatronics Research, Samsung Electronics Co., Ltd., Hwaseong-si, Gyeonggi-do, South Korea}
\address[orgy]{Electrical and Electronic Engineering Department, Yonsei University, Seoul, South Korea}

%

\begin{abstract}
The periocular area, which refers to the peripheral area of the ocular, is a valid biometric for situations where facial recognition is not possible due to occlusion or masking. However, periocular biometrics alone can reduce discriminative information, particularly in wild environments.
To address this, we propose Consistent Knowledge Distillation (CKD) that transfers discriminatory information from face images to periocular embeddings using temperature-based consistency. CKD achieves state-of-the-art results on challenging unconstrained periocular recognition benchmarks, improving performance by 49-54\% in relative gain. Furthermore, we provide insight into how CKD effectively extracts and transfers global inter-class relationship information by showing that CKD is equivalent to a learned-label smoothing approach with a novel sparsity-oriented regularizer.
\end{abstract}





\end{frontmatter}


\section{Introduction}

Biometrics is the physiological or behavioral unique characteristic of individuals such as fingerprint, face, iris, retina, etc. Besides the universally accepted fingerprint, biometric systems exploiting ocular traits, including the iris and retina, have also accomplished a significant breakthrough in the past decades. Recently, periocular-based biometrics, or peripheral area of ocular, encapsulating eyebrow, eyelid, eyelash, eye shape, tear duct, and skin texture, captures many unique features. Due to this, it has gained increasing attention as substitutional biometrics to face \cite{nigam2015ocular,rattani2017ocular,alonso2016survey}. Also, as a face component, it can be utilized compositionally for recognition in heterogeneous face domains such as sketch recognition \cite{liu2018composite, liu2020coupled, liu2021iterative}. 

\begin{figure}
\centering
\includegraphics[width=0.995\linewidth]{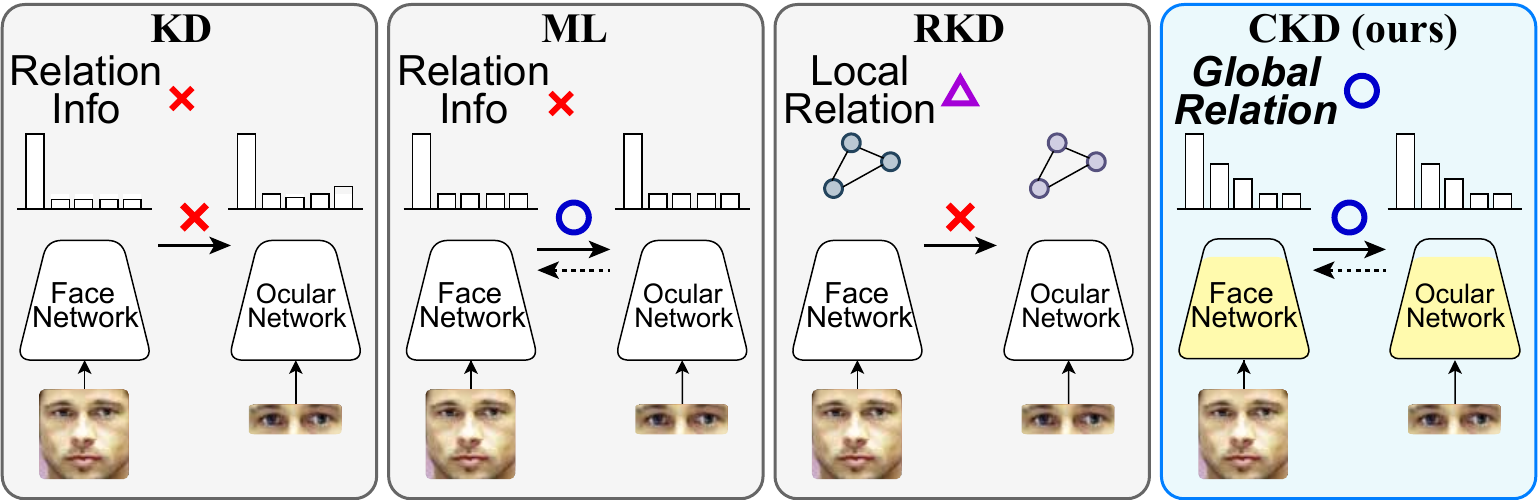}
\caption{
Comparison of different knowledge distillation methods. Knowledge distillation (KD) neither captures discriminative relationship information from face images nor effectively transfers it. Mutual learning (ML) only resolves information transfer. Relational knowledge distillation (RKD) can capture relationships between faces but only locally, and its transfer is ineffective. \textit{Our CKD can capture the global inter-class relationships between faces and effectively transfer that information to the periocular network.} Moreover, by consistency in feature layers via shared weights batch statistics, CKD extracts periocular features robust against identity-irrelevant attributes.
}
\label{fig:comp_methods}
\end{figure}

Periocular biometrics can be a valuable alternative to face for identity recognition, mainly when the facial features are not fully available due to occlusion, makeup, and plastic surgery \cite{alonso2016survey}. Under unconstrained environments, facial parts complementing periocular may likely involve attributes highly irrelevant to the individual identity traits. Moreover, due to the outbreak of the COVID-19 pandemic, individuals wearing masks are common, and masked face recognition started receiving attention \cite{huber2021mask}. These conditions strongly motivate an identity recognition system based solely on periocular biometric images.

The primary reason for the feasibility of using sole periocular biometrics for identity recognition is that the periocular region of a face contains meaningful details strongly corresponding to individual identity traits. Moreover, given a face dataset, a large-scale of periocular images is easily acquired by cropping the periocular region from face images with respect to the detected facial landmarks \cite{sun2013deep}.

Identification or verification of an individual solely based on one's periocular biometric, however, can be challenging, especially for weak hand-crafted models \cite{alonso2016survey,sun2013deep,park2010periocular,karahan2014identification,mahalingam2013lbp,cao2016fusion,xu2010robust}, due to absence of lower face and hair shape, inter-class discriminative traits may not be evident in periocular images even in human eyes. Moreover, in wild unconstrained settings, the periocular images possess diverse identity-irrelevant attributes with a large intra-class variation. The co-modulation of identity-irrelevant distracting factors and insufficient inter-class discriminant information are critical parts of periocular recognition.

To this end, deep neural networks (DNN) \cite{tiong2019periocular,long2015fully,proencca2017deep,zhao2018improving} can be a valid option due to their strong modeling capacity. Vanilla DNNs, however, are likely to overfit the training data, making their deployment to practical environments problematic. To mitigate such issues, \cite{tiong2019periocular,long2015fully,proencca2017deep,zhao2018improving} utilize face-periocular pairs as a type of data augmentation to both regularize the overfitting and complement scarce discriminative information of the sole periocular biometrics. However, these methods require face-periocular pair during deployment.

\begin{figure*}
\centering
\includegraphics[width=0.8\linewidth]{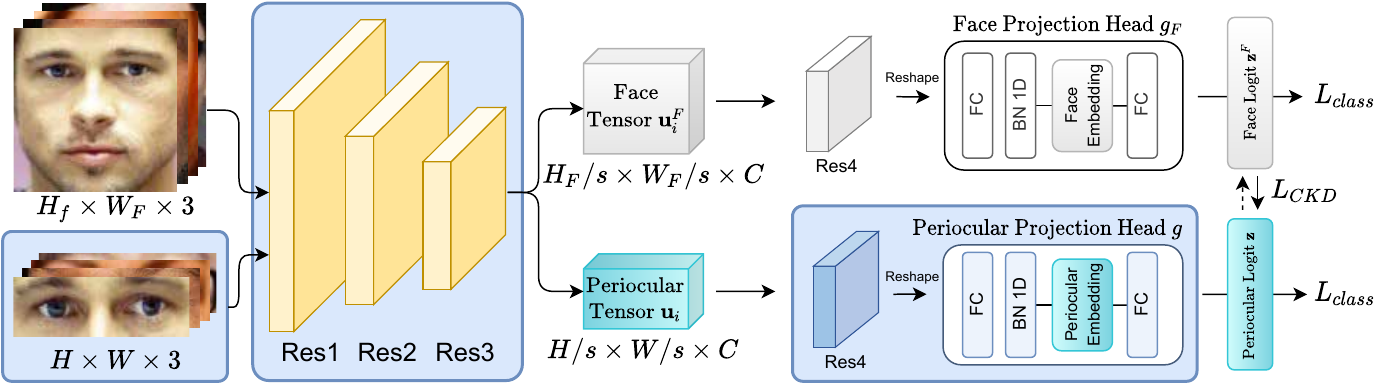}
\caption{The network architecture of CKD. A paired input contains a face and a periocular region feed forwarded to a shared-weights network, followed by respective projection heads. The whole network in the figure is trained, and only the colored area is used in the testing stage for periocular recognition.}
\label{fig:network}
\end{figure*}

On the other hand, \cite{jung2020glsr} supplements face images to the periocular network in a manner of knowledge distillation (KD). In \cite{jung2020glsr}, face images are used only during the training. The sole periocular image is streamed through the network for identification and verification for deployment. The work \cite{jung2020glsr} showcases that transferring discriminative information from face images can enhance periocular recognition.

We assert, however, that usage of KDs such as \cite{jung2020l2sr} can be suboptimal for our objective. In particular, as we show in the experiments, the application of naive KDs can neither fully extract profound inter-class discriminative information from the face nor effectively transfer that information to the periocular network. Moreover, the original KD is not meant to be robust against identity-irrelevant attributes at embedding learning levels.



To resolve these issues, we propose Consistent Knowledge Distillation that imposes temperature-based consistency between face and periocular network signals across prediction and feature layers. Imposing temperature-based consistency at the prediction layer enables
(1) extraction of profound inter-class discriminative relationship information from face images at a global scale and (2) effectual transfer of that structural information from face images to the periocular network. Notably, the prediction consistency stimulates the face network to store inter-class relationships in the non-target posterior where the non-target posterior indicates class prediction probability not corresponding to the target class identity. Then, this relationship information is transferred to the periocular network again by prediction consistency.



On the other hand, consistency is substantiated at the feature layers by sharing network weights and batch-norm statistics, which enables (3) the periocular network to learn features consistent across both face and periocular and, therefore, robust against identity-irrelevant features.
Overall, the effective transfer of consistent knowledge over prediction and feature layers enables the periocular network to produce robust embeddings for periocular recognition in wild environments.

We theoretically validate the principal parts of CKD mentioned in the above points (1) and (2). In the meantime, we discover that CKD is equivalent to learned-label smoothing \cite{jung2020l2sr} accompanied by a novel sparsity-oriented regularizer. We prove that this regularizer is necessary and sufficient for profound inter-class relationship learning from face images. The theory is accompanied by extensive empirical support.


Unlike the typical KD, CKD is a single-stage training regimen without pretraining and involves only one hyperparameter. With much simpler training, CKD achieves state-of-the-art performance across diverse unconstrained periocular recognition datasets with significant gaps to the typical KD application.

The contributions of our work are summarized as follows:
\begin{itemize}
\item
To tackle the periocular recognition problem, we propose a novel CKD method that enhances the periocular embeddings by the relational information in face images. 
CKD imposes \textit{temperature-based consistency} between the face and periocular network. This extracts inter-class information between whole train identities from faces and transfers it to the periocular network, improving its robustness. (Sec.~\ref{sec:method})

\item
Our theoretical analysis demonstrates that the temperature-based consistency loss used in CKD is equivalent to a learned-label smoothing approach that incorporates a novel sparsity-oriented regularizer. This equivalence explains (1) how CKD can extract inter-class relationships from face images and (2) the effective transfer of this relational information to the periocular network. (Sec.~\ref{sec:theory})

\item
CKD is a single-stage knowledge distillation method with a single hyperparameter. It achieves state-of-the-art performance over six standard periocular recognition benchmark datasets, improving its periocular baseline by 49-54\% relative gains. In addition, it consistently outperforms all other periocular recognition methods by significant margins. (Sec.~\ref{sec:exp})


\end{itemize}




\section{Related Works}
This section reviews periocular recognition literature covering hand-crafted and deep learning-based methods, followed by various knowledge distillation techniques relevant to our proposed method.

\subsection{Periocular Biometric}

Early works in periocular biometric recognition are primarily based on hand-crafted texture descriptors such as Histogram of Orientation and Gradient (HOG), Scale Invariant Feature Transform (SIFT), and Local Binary Pattern (LBP) \cite{park2010periocular,karahan2014identification,mahalingam2013lbp}. Further studies combined several hand-crafted texture descriptors by fusing, for example, the Gabor Filter with HOG and LBP \cite{cao2016fusion}. A survey study \cite{xu2010robust} examines a variety of texture descriptors based on hand-craft, including Walsh Hadamard product LBP (WLBP), Laws Masks, Discrete Cosine Transform (DCT), Discrete Wavelet Transform (DWT), Force Field transform, Speed-Up Robust Features (SURF), Gabor filters, and Laplacian of Gaussian (LoG) filters. Although some accuracy performance gain is observed by utilization of those descriptors, they are shown to suffer from external distractors (aging degradation, pose invariance, and illumination changes).
The earlier studies reveal that hand-crafted features are ineffective for periocular recognition in unconstrained wild environments.

The recent works in periocular recognition utilize deep Convolutional Neural Networks (CNN) to \textit{learn features} rather than hand-crafting it. For instance, \cite{proencca2017deep} trains a CNN on augmented artificial periocular image samples crafted by interchanging ocular regions of different identities. 
On the other hand, motivated by the human visual attention mechanism that focuses on the eyebrow and ocular area of the periocular image, \cite{zhao2018improving} devises a two-stage scheme, where
The CNN learns from periocular images under the guidance of pretrained attention maps. However, the network of \cite{zhao2018improving} fails under unconstrained environments due to severe dependency on the pretrained attention map.
\cite{zanlorensi2020unconstrained} adopted generative adversarial networks (GAN) \cite{goodfellow2014generative} to remove the noisy attributes such as glasses and eye gaze, leaving only the discriminative attributes meaningful to recognition purposes. \cite{luo2021deep} opts iris to fuse it with periocular to enhance the attention mechanism of the network. \cite{borza2023adaptive} equips an attention-based local branch to focus more on the salient periocular features. These works, however, have not considered transferring informative information from data such as face to improve performance. \cite{talreja2022attribute} utilizes soft biometrics of subjects to improve periocular recognition performance, but their method is trained in 3 stages.

Near-infrared (NIR) periocular recognition also has started to gain attention for periocular recognition. \cite{hwang2020near} fuses globally pooled features with local features extracted from NIR periocular images. \cite{behera2020variance} aims to extract identity-relevant features by cross-matching the NIR periocular image with the corresponding RGB image under the high-level attention module and variance-guided loss function. However, the NIR-based methods \cite{hwang2020near,behera2020variance} underperform in unconstrained settings. 

Considering the periocular recognition in the wild challenge, \cite{tiong2019periocular} outlines a dual-stream CNN that encodes the Orthogona-Local Binary Coded Pattern descriptor – a composition of the well-established Local Binary Pattern and the Local Ternary Patterns (LTP).
On the contrary, \cite{jung2020l2sr} tackles this problem by regularizing the ground truth labels with a pre-task periocular network. In contrast, the periocular network of \cite{jung2020glsr} stands out from the state of the arts by predicting the soft targets learned from faces. The course of predicting the soft targets of the face in \cite{jung2020glsr} has an equivalence to knowledge transfer in KD \cite{yuan2020revisiting}. \cite{ng2022conditional} trains a joint face and periocular network conditioning face to periocular using a face-periocular contrastive loss. These works transfer knowledge from face to periocular but don't consider the important global relational information between face and periocular, which is very effective.

\subsection{Knowledge Distillation (KD)}
Earlier works \cite{hinton2015distilling} in KD are focused on model compression; KD transfers knowledge from a large pretrained model (teacher) to a new small compressed model (student). The conventional scheme for knowledge transfer is by mimicking the logit of larger network logit. Still, other KD methods showcase knowledge transfer by mimicking different types of network signals in the teacher network, such as attention map \cite{zagoruyko2016paying}, intermediate layer features, and/or other structural/statistical outputs \cite{zhang2019your,tung2019similarity,park2019relational,passalis2018learning} derived from the network.


\noindent
\textbf{KD to enhance biometric recognition.}
{

KD used in the biometrics is presented in \cite{wang2018discover}, which improves face recognition tasks by minimizing the L2 distance of normalized features between teacher and student penultimate layers. Especially, KD in biometrics is widely used for low-resolution face recognition. \cite{karlekar2019deep} enhances low-resolution face recognition by distilling knowledge from a network trained on high-resolution face images. Moreover, \cite{ge2018low} selectively distills information from teachers trained with high-resolution images to low-resolution student networks. \cite{ge2020efficient} distills in two steps where the first is a cross-dataset distillation from a high-resolution teacher trained with the large private dataset to an additional branch from the same backbone distilled with the public dataset, and the second is supervising the low-resolution student with the teacher branch trained with the public dataset. 
}

\noindent
\textbf{Other KD works relevant to our model CKD.}
The analytical works \cite{tang2020understanding,yuan2020revisiting,kobayashi2022extractive,kim2021comparing} observed that KD is closely related to label smoothing. On top of these observations, we, for the first time, theoretically verify that KD is \textit{precisely equal} to label smoothing with \textit{a sparsity-oriented regularizer}. The regularizer prevents over-smoothing of predictions and helps the network capture relational information of subject identities. The property of the regularizer motivates us to apply KD on both face and periocular networks (hence the name CKD), thereby extracting relational information from the face and effectively transferring it to the periocular network.

Mutual Learning (ML) \cite{zhang2018deep} is related to our CKD in the aspect of posterior alignment by bi-directional distillation losses (Fig.~\ref{fig:comp_methods}), but the original works utilize only a single domain dataset. Direct application of ML to our problem is feasible. However, due to the absence of the smoothing factor of the prediction logit, the face network of ML would not learn relational information well. We empirically show in Sec.~\ref{sec:exp} that ML largely underperforms compared to our CKD.

Relational Knowledge Distillation (RKD) transfers metric (or angular) relations between embedding vectors from one network to another (Fig.~\ref{fig:comp_methods}). Applying RKD to our objective (periocular recognition enhancement) is also an option. RKD, however, does not involve bi-directional knowledge transfer and captures relational information only at a local scale. Thus, knowledge transfer of RKD is weak, and relational information extracted from face images by RKD can be insufficient. We experimentally found that local relationship transfer by RKD initially meant for a single data type degrades the performance due to data discrepancy between face and periocular.

\section{Method: Consistent Knowledge Distillation}
\label{sec:method}

CKD enhances the periocular network by transferring the knowledge of the face consistent with that of the periocular with simultaneous training of the face network that serves as a teacher. As CKD imposes \textit{temperature-based consistency} between predictions of periocular and face networks (Sec.~\ref{sec:method_pred}), the profound inter-class relationship is extracted from face images (Theorem \ref{thm:ckd_smooth}), and is effectively transferred to the periocular network via the prediction layer (Proposition \ref{thm:align}). Moreover, consistent knowledge sharing in the feature layers (Sec.~\ref{sec:method_fea}) improves the periocular network for robust recognition in wild environments.

\subsection{Base Settings of CKD}
Before delving into the specific details of the distillation mechanisms of CKD, we first describe the basic settings of CKD. The train dataset $\{ (\mathbf{x}_i, \mathbf{x}^F_i, y_i) \}_{i=1}^N$ consists of a pair of periocular image $\mathbf{x}_i {\in} \mathbb{R}^{H \times W \times 3}$, face image $\mathbf{x}^F_i {\in} \mathbb{R}^{H_F \times W_F \times 3}$, and the corresponding identity label $y_i {\in} \{ 1, \dots, K \}$. The face and periocular networks compute posterior distributions of identity, $\mathbf{p}_i {:=} \softmax(\mathbf{z}_i / \tau)$ and $\mathbf{p}^F_i {:=} \softmax(\mathbf{z}^F_i / \tau)$, which are defined as the softmax outputs of the corresponding logit vectors $\mathbf{z}_i \in \mathbb{R}^K$ and $\mathbf{z}^F_i \in \mathbb{R}^K$ of $\mathbf{x}_i$ and $\mathbf{x}^F_i$, respectively. Here, $\tau$ is the temperature that serves as a smoothing factor of the logit.

For the posteriors to model the class identity distributions of face and periocular, they are trained by minimizing the classification loss based on the standard cross entropy
\begin{equation}
L_{class}=
-  \log p_{i,y_i}
-  \log p^F_{i,y_i}
\end{equation}
where $p_{i,k}$ is the $k$-th component of the posterior $\mathbf{p}_i = (p_{i,1}, \dots, p_{i,K}) \in \mathbb{R}^K$, and likewise for $p^F_{i,k}$. We note that during the computation of the classification loss $L_{class}$, the temperature $\tau$ of the logits is set to $1$.

\subsection{CKD at Prediction Layer}
\label{sec:method_pred}
At the prediction layer, CKD transfers consistent knowledge of the face to the periocular network by minimizing the symmetric statistical distance between the posterior outputs of the face and periocular. In particular, the CKD model minimizes \textit{the temperature-based consistency loss} \cite{hinton2015distilling}:
\begin{equation}
\label{eq:ckd_loss}
L_{CKD} =
\tau^2 D_{KL}( \mathbf{p}^F_i \parallel \mathbf{p}_i ) + \tau^2 D_{KL}( \mathbf{p}_i \parallel \mathbf{p}^F_i )
\end{equation}
where $D_{KL}$ is the Kullback–Leibler divergence between two probability distributions. We note that in the computation of the KL divergence $D_{KL}(\mathbf{p}^F_i \parallel \mathbf{p}_i)$, stop-gradient operation is applied to its first variable $\mathbf{p}^F_i$, and likewise for the first variable $\mathbf{p}_i$ of $D_{KL}(\mathbf{p}_i \parallel \mathbf{p}^F_i)$.

Minimizing the bi-directional KL divergences in the CKD loss $L_{CKD}$ effectuates precise alignment between face and periocular posteriors, making both networks capture knowledge relevant to subject identity and their relational information. Technically, the CKD loss accompanied by the classification loss serves to smooth the face posterior $\mathbf{p}^F_i$ and regularize it to capture \textit{global inter-class relationship} from face images (Theorem \ref{thm:ckd_smooth} and Proposition \ref{thm:reg}).
Moreover, through the posterior alignment by minimization of $L_{CKD}$ (Proposition \ref{thm:align}), the knowledge of the face is \textit{transferred to a full extent} to the periocular network. Overall, the CKD loss effectively transfers inter-class relationship information from face images to the periocular network, thereby enhancing periocular recognition in the wild.


\subsection{CKD at Feature Layers}
\label{sec:method_fea}
Face and periocular images of a corresponding identity share the same intrinsic image features to a large extent. Thus, extracting their image features based on independent extraction mechanisms would be inefficient and possibly vulnerable to overfitting (Table~\ref{table:model_ablation}).

We also propose sharing consistent knowledge in the feature extraction layers to prevent such disadvantages.
As shown in Fig.~\ref{fig:network}, at the early stage of feature extraction, both face and periocular images are processed by the shared convolution kernels and shared batch normalization layers. Thus, consistency is realized at the feature layers by applying the same network weights. Then, at the final stage of feature extraction, two different projection heads are used to face and periocular intermediate layer features, respectively, to capture features specific to each different biometric characteristic of the face and periocular.

Technically, given a batch $\{ (\mathbf{x}_i, \mathbf{x}^F_i) \}_{i=1}^B$ of the face and periocular image pair, their features are extracted by a shared network $f$
\begin{equation}
\label{eq:shared}
\mathbf{u}_i = f(\mathbf{x}_i), \quad \mathbf{u}^F_i = f(\mathbf{x}^F_i)
\end{equation}
where $f = f(\cdot; \theta)$ consists of shared weights (SW) $\theta$.
Here, the shared network $f$ extracts consistent features from face and periocular images using convolutions and batch normalizations.


During the extraction process Eq.~\eqref{eq:shared} by the shared network $f$, batch normalization is applied to the face and periocular pair with shared batch statistics (SBS). In particular, the batch statistics are computed on the reshaped batch by treating $[\mathbf{x}_i; \mathbf{x}^F_i]_{i=1} \in \mathbb{R}^{(B \times [(\frac{H}{s}  \times \frac{W}{s}) + (\frac{H_F}{s} \times \frac{W_F}{s})]) \times C}$ as a stack of $C$-dimensional vector. (Here, we assume the channel sizes $C$ and $C_F$ of face and periocular are the same, $\mathbf{x}_i$ and $\mathbf{x}_i^F$ are intermediate layer features for the time being, and $s$ is a stride factor of that intermediate layer.) Thus, batch mean $\boldsymbol{\mu} \in \mathbb{R}^C$ and deviation $\boldsymbol{\sigma} \in \mathbb{R}^C$ are
\begin{equation}
\mu_c =
\frac{
\sum_{i, h, w} x_{i, h, w, c} +
\sum_{i, h, w} x^F_{i, h, w, c}
}{
B \cdot [(\frac{H}{s} \cdot \frac{W}{s}) + (\frac{H_F}{s} \cdot \frac{W_F}{s})]
}
\end{equation}
for each $c=1,\dots,C$, and likewise for the channel-wise deviation $\sigma_c$.  Then, the final embeddings are computed by applying separated projection head $g$ and $g_F$; namely, $\mathbf{z}_i = g(\mathbf{u}_i)$ and $\mathbf{z}^F_i = g_F(\mathbf{u}^F_i)$ to capture features specific to each biometrics.

Overall, CKD substantiates consistency at feature layers by imposing shared weights (SW) and shared batch statistics (SBS), thereby extracting periocular embedding consistent with face images and robust for periocular recognition in wild environments.

\subsection{Full Objective of CKD}
Overall, the CKD model is trained by minimizing the following full objective:
\begin{equation}
L_{full} =
L_{class} + L_{CKD}
\end{equation}
where $L_{CKD}$ given as in Eq.~\eqref{eq:ckd_loss} imposes consistent knowledge at the classification layer.
In the computation of these losses,
the shared network $f$ processes face and periocular images through SW and SBS, extracting consistent features from them.

The CKD model contains only one hyperparameter, the temperature $\tau$ that smooths the logits in the CKD loss $L_{CKD}$. The temperature $\tau$ in the computation of $L_{CKD}$ is fixed to $2.5$ unless specified otherwise.

\subsection{Identification and Verification}
Upon training completion, the face projection head is dispensed. In other words, the identification and verification involve only the shared-weight networks and the periocular projection head, and only periocular images are used.

The identification task requires gallery and probe sets. The gallery is collected by forwarding the gallery images to the trained periocular network and extracting the periocular embedding vectors $\mathbf{v}_G$ paired with each of their identities. Then, the probe images are forwarded to the network, and feature vectors $\mathbf{v}_P$ is compared with every $\mathbf{v}_G$ with the cosine similarity. The identity with the highest similarity is the prediction.

Unlike 1-to-N matching in the identification task, the 1-to-1 verification task determines if the two periocular images of a pair belong to the same (positive) or different (negative) identities. Paired images are computed with the cosine similarity, and a Receiver Operating Characteristic (ROC) curve is obtained with an Equal Error Rate (EER) based on the calculated similarity values.

\section{Theoretical Analysis}
\label{sec:theory}
In this section, we observe the following theoretical principles: (1) The total objective of CKD can be decomposed into the learned-label smoothing part and a novel sparsity-oriented regularizer. The prediction smoothing part with learned labels increases the entropy of the posterior, enabling the posterior to capture the discriminative relationship. On the other hand, the sparsity-oriented regularizer prevents over-smoothing of the prediction and makes the non-target posterior capture class relationships more meaningful. (2) Prediction consistency in CKD aligns the face and periocular posteriors, maximizing the information transfer from the face to the periocular network. We empirically validate our theoretical claims.

{

We highlight that, unlike \cite{yuan2020revisiting}, we are first to rigorously show that the KD loss is precisely equal to label smoothing with a novel regularizer (Thm.~\ref{thm:ckd_smooth}). \cite{yuan2020revisiting}, on the other hand, does not give rigorous derivation. Moreover, by geometrically analyzing the graph of the regularizer, we derive that the novel regularizer promotes sparsity on the non-target units of logit (Prop.~\ref{thm:reg}, thereby enabling the model to learn inter-class relationship information on the logit. We are the first to show how KD promotes inter-class relationship learning rigorously.
}

Below, we regard the label $\mathbf{y}$ as a one-hot vector or an integer $y \in \{1,\dots,K\}$, depending on the context. 

\subsection*{(1) CKD is equivalent to label smoothing of the face and periocular network predictions with a novel regularizer that helps the network capture relational information of subject identities.}

We first show that CKD can be decomposed into prediction smoothing and regularization.

\begin{thm}
\label{thm:ckd_smooth}
Let $\mathbf{p}_\tau$ denote the softmax probability of the periocular logit $\mathbf{z}$ divided by $\tau$, $p_{\tau,k} = e^{z_k/\tau} / \sum_i e^{z_i/\tau}$, and likewise for face, $\mathbf{p}^F_\tau$. Let $\mathbf{p}=\mathbf{p}_1$ and $\mathbf{p}^F = \mathbf{p}^F_1$ be the softmax posteriors with $\tau=1$. Then, $L_{full}$ is equal to 
\begin{equation}
\label{eq:loss_equiv}
L_{full}
= H(\widetilde{\mathbf{y}}, \mathbf{p}^F) + H(\widetilde{\mathbf{y}}^F, \mathbf{p}) 
+ \frac{\tau}{1 + \tau} \left(R(\mathbf{z}^F) + R(\mathbf{z})\right)
\end{equation} 
up to scale 
where $H(\mathbf{y}, \mathbf{p}) = - \sum_k y_k \log p_k $ is cross entropy, $\widetilde{y}$ is a  smooth label
$
\widetilde{\mathbf{y}} = \frac{\mathbf{y} + \tau \mathbf{p}_\tau}{1 + \tau}
$
with the similarly defined smooth label $\widetilde{\mathbf{y}}^F = \frac{\mathbf{y} + \tau \mathbf{p}^F_\tau}{1 + \tau}$, and 
\begin{equation}
R(\mathbf{z}) = - \log \frac{\sum_i e^{z_i}}{\left(\sum_i e^{z_i/\tau} \right)^\tau}
\end{equation}
is a regularizer with a similar regularizer $R(\mathbf{z}^F) = - \log (\sum_i e^{z^F_i} / (\sum_i e^{z^F_i / \tau})^\tau)$
\end{thm}

To prove the above theorem, we utilize a technical lemma.

\begin{lem}
\label{thm:ckd_smooth_lemma}
We have
$ H(\mathbf{p}^F_\tau, \mathbf{p}_\tau) = \frac{1}{\tau} H(\mathbf{p}^F_\tau, \mathbf{p}) + \frac{1}{\tau} R(\mathbf{z})$,
and similarly for $H(\mathbf{p}_\tau, \mathbf{p}^F_\tau)$.
\end{lem}

\begin{proof}[Proof of Lemma \ref{thm:ckd_smooth_lemma}]
Let $q_k = e^{z_k}$. Then, 
$ H(\mathbf{p}^F_\tau, \mathbf{p}_\tau) = - \sum_k p_{\tau, k}^F \log \frac{q_k^{1/\tau}}{\sum_i q_i^{1/\tau}}$.
We simplify the log term
\begin{equation}
\log \frac{q_k^{1/\tau}}{\sum_i q_i^{1/\tau}} = \log \frac{q_k^{1/\tau}}{\sum_i q_i^{1/\tau}} \frac{\left(\sum_i q_i\right)^{1/\tau}}{\left(\sum_i q_i\right)^{1/\tau}},
\end{equation}
which is equal to 
$ \frac{1}{\tau} \log \frac{q_k}{\sum_i q_i} + \frac{1}{\tau} \log \frac{\sum_i q_i}{\left(\sum_i q_i^{1/\tau}\right)^\tau}$.
Therefore, substituting the above expression in $H(\mathbf{p}^F_\tau, \mathbf{p}_\tau)$ and noting $q_k = e^{z_k}$, we obtain the desired result.
\end{proof}

\begin{proof}[Proof of Theorem \ref{thm:ckd_smooth}]
Note that
$ H(\mathbf{y},\mathbf{p}) + \tau H(\mathbf{p}^F_\tau, \mathbf{p}) 
= (1+\tau) H \left(
\frac{\mathbf{y} + \tau \mathbf{p}^F_\tau}{1+\tau}, \mathbf{p}
\right) $
by the linearity of $H$ in its first variable,
and thus 
$ H(\mathbf{y},\mathbf{p}) + \tau H(\mathbf{p}^F_\tau, \mathbf{p}) = (1+\tau)H(\widetilde{\mathbf{y}}^F, \mathbf{p})$.
Similarly, $H(\mathbf{y},\mathbf{p}^F) + \tau H(\mathbf{p}_\tau, \mathbf{p}^F) = (1+\tau) H(\widetilde{\mathbf{y}}, \mathbf{p}^F)$.
Therefore, by Lemma \ref{thm:ckd_smooth_lemma}, minimizing the full objective $L_{full}$ of CKD is equivalent to minimizing
\begin{equation}
H(\mathbf{y}, \mathbf{p}^F) + H(\mathbf{y}, \mathbf{p}) + \tau^2 H(\mathbf{p}_\tau, \mathbf{p}^F_\tau) + \tau^2 H(\mathbf{p}^F_\tau, \mathbf{p}_\tau),
\end{equation}
which, after rearrangement, is proportional to
\begin{equation}
H( \widetilde{\mathbf{y}}, \mathbf{p}^F ) + \frac{\tau R(\mathbf{z}^F)}{1+\tau} 
+ H(\widetilde{\mathbf{y}}^F, \mathbf{p}) + \frac{\tau R(\mathbf{z})}{1+\tau},
\end{equation}
deducing the desired equation.
\end{proof}

Theorem \ref{thm:ckd_smooth} states that the full objective of the CKD model can be decomposed into label smoothing with learned labels and regularizers $R(\mathbf{z})$ and $R(\mathbf{z}^F)$. The label smoothing and regularization mechanisms are applied to the face and periocular predictions.
We carefully analyze each of these mechanisms:

\begin{prop}
\label{thm:smooth_label}
When $\tau \to \infty$,
the smooth label $\widetilde{\mathbf{y}}$ converges to the uniform noise $\mathbf{1/K} = (1/K, \dots, 1/K) \in \mathbb{R}^K$;
$
\lim_{\tau \to \infty} \widetilde{\mathbf{y}} = \lim_{\tau \to \infty} \mathbf{p}_\tau = \mathbf{1/K},
$
and the smooth label $\widetilde{\mathbf{y}}$ is asymptotically equal to $\mathbf{p}_\tau$. Similar property holds for $\widetilde{\mathbf{y}}^F$ and $\mathbf{p}^F_\tau$.
\end{prop}

\begin{proof}
Note that $\mathbf{y}/(1+\tau) \to \mathbf{0}$, and $\tau \mathbf{p}_\tau / (1+\tau) = \mathbf{p}_\tau / (1/\tau + 1) \to \mathbf{1/K}$ as $\mathbf{p}_\tau \to \mathbf{1/K}$ with $\tau \to \infty$.
\end{proof}

\begin{cor}
\label{thm:entropy}
If $H(\widetilde{\mathbf{y}}, \mathbf{p}^F)$ is the minimum, then the entropy of the posterior converges to its maximum;
$
\lim_{\tau \to \infty} H(\mathbf{p}^F) = \log \frac{1}{K}
$
where $H(\mathbf{p}^F) = \sum_k p^F_k, \log p^F_k$.
\end{cor}

Proposition \ref{thm:smooth_label} with its corollary indicates that when the temperature $\tau$ is large, the face network $\mathbf{p}^F$ \textit{learns more from the learned periocular posterior} $\mathbf{p}_T$ than the fixed one-hot label $\mathbf{y}$, and the face network prediction learns from \textit{more smooth labels}. 

With an overly large temperature, the prediction posterior may become over-smooth (Fig.~\ref{fig:ablate_reg}a). Moreover, smoothness alone does not guarantee the posterior captures profound relationship information. The sparsity-oriented regularizer prevents this possible negative impact by imposing sparsity constraint over the prediction posterior:


\begin{prop}
\label{thm:reg}
$R(\mathbf{z}) \geq 0$ achieves a minimum when $\mathbf{p}$ is sparse; namely, it attains the global minimum by
\begin{equation}
\lim_{z_k \to \infty} R(\mathbf{z}) =
\underset{
\mathbf{z} \to (-\infty, \dots, -\infty, z_k, -\infty, \dots, -\infty)
}{
\lim
}
R(\mathbf{z}) = 0
\end{equation}
For any fixed $k$.

On the other hand, in terms of the temperature $\tau$, $\lim_{\tau \to \infty} R(\mathbf{z}) = \infty$, and $R(\mathbf{z})=1$ is a constant when $\tau=1$.

The same properties hold for $\mathbf{z}^F$.
\end{prop}

\begin{proof}
Note that $(\sum_i e^{z_i/\tau})^\tau \geq \sum_i e^{z_i}$, showing that $R(\mathbf{z}) \geq 0$. 
On the other hand, note that minimizing $R(\mathbf{z})$ is equivalent to maximizing $\sum_k e^{z_k} / (\sum_i e^{z_i/\tau})^\tau$.
Considering its summand $e^{z_l} / (\sum_i e^{z_i/\tau})^\tau$, 
$
e^{z_l} / (\sum_i e^{z_i/\tau})^\tau \to 0
$
for $l \neq k$ and
$
e^{z_k} / (\sum_i e^{z_i/\tau})^\tau \to 1
$
if $z_k \to \infty$ with fixed $z_l$ ($l \neq k$). If $z_l \to -\infty$ with $z_k$ fixed, then 
$
\sum_k e^{z_k} / (\sum_i e^{z_i/\tau})^\tau \to e^{z_k} / (e^{z_k/\tau})^\tau = 1.
$

On the other hand, $e^{z_k/\tau} \to 1$ as $\tau \to \infty$. Hence, with $K \geq 2$, $\sum_{i=1}^{K} e^{z_i/\tau} >1$ for sufficiently large $\tau$. Thus, $\exp(-R(\mathbf{z}))= \sum_k e^{z_k} / (\sum_i e^{z_i/\tau})^\tau \to 0$, which implies $R(\mathbf{z}) \to \infty$.
\end{proof}

\begin{figure}
\centering
\includegraphics[width=0.9\linewidth]{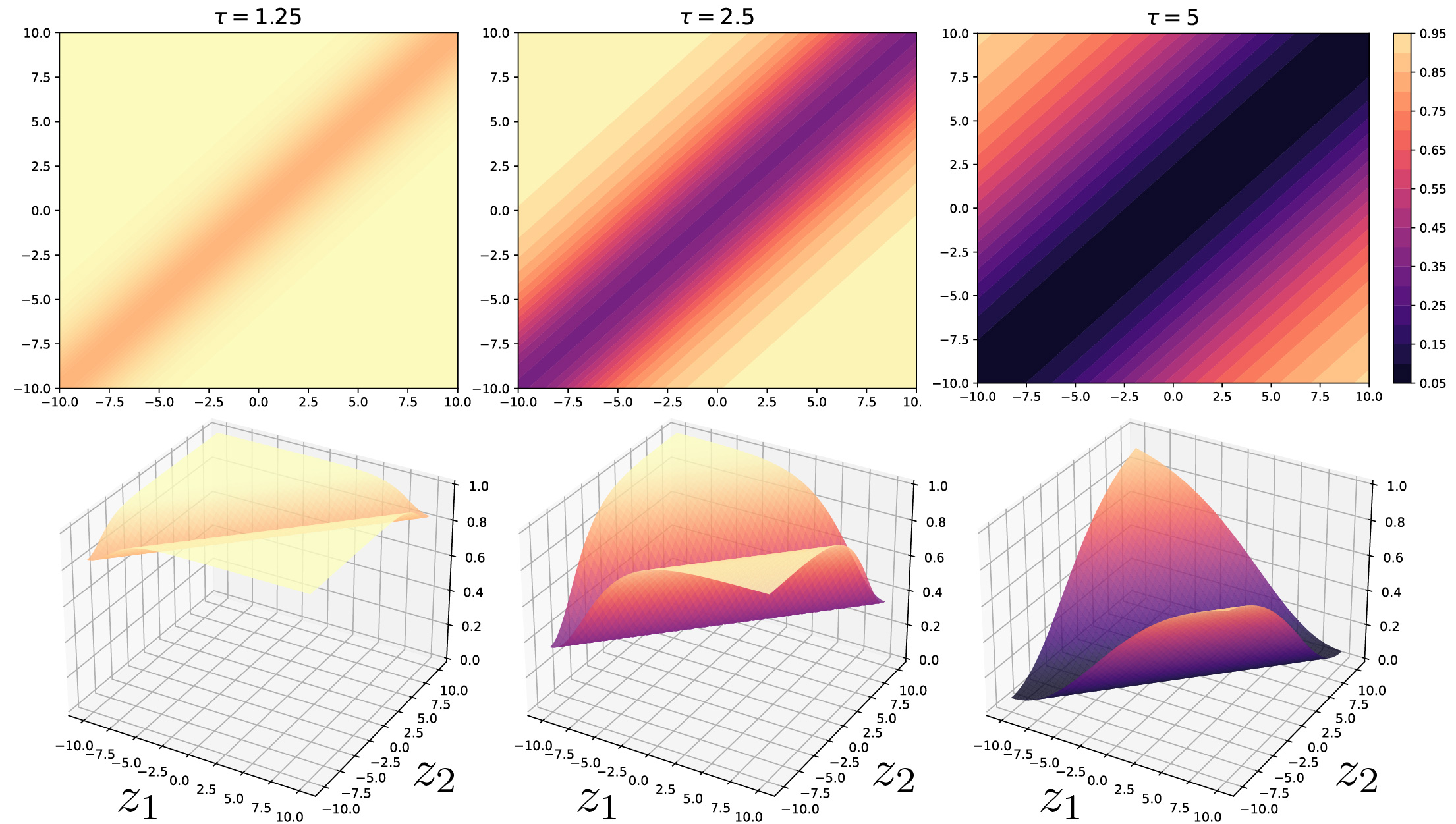}
\caption{
The visualization of of the regularizer $R(\mathbf{z})$ by observing the exponent of its negative $\sum_{k=1}^K e^{z_k} / (\sum_{k=1}^K e^{z_k/\tau})^\tau = \exp(-R(\mathbf{z}))$ with $K {=} 2$ and $\tau =$ 1.25, 2.5, and 5.
(First row) Its 2-D heatmap visualization (Second row) is the corresponding 3-D visualization. The plots show that the regularizer is minimized when either of $z_k$ is maximized, and the other $z_i$ is minimized; namely, $R(\mathbf{z})$ is minimized when the softmax $\mathbf{p}$ (posterior) of the logit $\mathbf{z}$ converges to one-hot. Hence, $R(\mathbf{z})$ \textit{as a regularizer prevents over-smoothing of the predictions} $\mathbf{p}$. In terms of the temperature $\tau$, on the other hand, increasing $\tau$ decreases the exponent, thereby increasing the upper bound of $R(\mathbf{z})$. Thus, a large temperature $\tau$ enhances the impact of regularizer $R(\mathbf{z})$. The theoretical observation of $R(\mathbf{z})$ is given in Proposition \ref{thm:reg}.
}
\label{fig:regularizer}
\end{figure}

\begin{figure}
\centering
\includegraphics[width=0.75\linewidth]{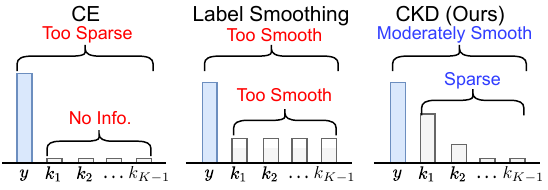}
\caption{
The softmax outputs of different models.
By cross entropy (CE), the face network learns sparse predictions unless it underfits. Label smoothing alone may over-smooth the prediction and not necessarily capture the inter-class relationships. Our CKD, however, smoothes the prediction only to a moderate degree and captures an inter-class relationship in the non-target posterior $p(k \neq y | \mathbf{x}^F)$.
}
\label{fig:effect_reg}
\end{figure}

\noindent
\textbf{The role of the regularizer.}
As depicted in Fig.~\ref{fig:regularizer}, the regularizer $R(\mathbf{z}^F)$ seeks the sparse prediction logit $\mathbf{z}^F$ of the face network by increasing the $k$-th unit of logit and corresponding posterior for \textit{general class index} $k$. Here, the posterior $p(k | \mathbf{x}^F)$ denotes the $k$-th unit of softmax output of $\mathbf{z}^F$. When the indices $k \neq y$ are non-target ones, it can increase the sparsity of non-target posterior $p(k \neq y | \mathbf{x}^F)$, storing meaningful inter-class relationship of face images in the non-target posterior (Fig.~\ref{fig:effect_reg}). Moreover, inclination towards the sparse logit by the regularizer prevents over-smoothing of the prediction caused by $H(\widetilde{\mathbf{y}}, \mathbf{p}^F)$. Experiments in Sec.~\ref{sec:exp_anal_reg} thoroughly verify the theoretical observations.


Minimizing the full objective loss of CKD enables the face network prediction to capture a profound inter-class relationship within the moderate degree of posterior entropy.

\noindent
\textbf{The effect of temperature $\tau$.} 
Based on the above theoretical properties, increasing the temperature of the logit increases the impacts of both label smoothing and the sparsity-oriented regularizer. Since their roles are opposite, a too-large temperature escalates their conflict. Therefore, the selection of a moderately large temperature is crucial. We found that $\tau{=}2.5$ works the best in our experiments.

\subsection*{(2) Minimization of the CKD loss induces precise alignment between face and periocular network posteriors}
\label{sec:theory_transfer}

The CKD loss consists of bi-directional KL divergences. Minimization of these divergences attains the minimum, although the two networks, the face, and periocular networks, learn from different datasets where one has limited information:

\begin{prop}
\label{thm:align}
Assume the face and periocular networks have sufficient model complexities. Then $L_{CKD}$ can be minimized, in which case $\mathbf{p} = \mathbf{p}^F$.
\end{prop}

\begin{proof}
Since the face network has sufficient model complexity, there is an update rule that $D_{KL}(\mathbf{p} \parallel \mathbf{p}^F) \to 0$. If converged, then $\mathbf{p}^F = \mathbf{p}$, and  $D_{KL}(\mathbf{p}^F \parallel \mathbf{p}) = 0$, implying $L_{CKD}=0$.
\end{proof}

Hence, minimizing the CKD loss induces precise alignment between the face and periocular posteriors.



\noindent
\textbf{Overall Implication of CKD.}
Due to the hidden sparsity-oriented regularizer in CKD (Theorem \ref{thm:ckd_smooth}), the face network extracts profound inter-class relationship information from whole identities in its prediction layer. Then, by the accurate alignment between the face and periocular posteriors, the inter-class relationship extracted from face images is effectively transferred to the prediction layer of the periocular network.


\noindent
\textbf{Differences to Standard KD in Theoretical Aspects.}
The vanilla KD starkly contrasts our proposed CKD in the above theoretical aspects. The standard KD transfers knowledge from the face to the periocular network without consistency. Due to the lack of consistency, the following can be observed: (1') Only the periocular network learns from smooth labels. Thus, the face network may overfit and does not extract relationship information from face images (2'). Due to the absence of bi-directional KL terms, the distillation loss can be suboptimal (Remark in Sec.~\ref{sec:theory_transfer}), so the knowledge transfer is also ineffective.

\section{Evaluation and Analysis}
\label{sec:exp}

\begin{table*}
\centering
\caption{Rank-1 identification rates and verification EER for six benchmarking datasets and their averages for the ablation study. Best performances are highlighted in bold font.
}
\resizebox{0.55\linewidth}{!}{   
\begin{tabular}{c l c c c c c c c c}
\hline
~ & ~ & Eth & PF & FS & IW & AR & YTF & Avg & Gain \\ 
\hline
\multirow{7}{*}{\rotatebox[origin=c]{90}{Identification}} & Face CE & 97.48 & 99.19 & 98.66 & 90.18 & 91.74 & 78.52 & 92.63 & 100 \\ 
~ & CE & 92.82 & 95.78 & 96.44 & 77.63 & 93.5 & 56.78 & 85.49 & 0 \\ 
~ & $L_{class}$+SW+SBS & 93.21 & 96.14 & 96.18 & 78.63 & 94.59 & 57.57 & 86.05 & 8 \\ 
~ & $L_{class}{+}L_{F2P}$+SW+SBS & 93.09 & 95.92 & 96.48 & 77.81 & 94.39 & 59.52 & 86.2 & 10 \\ 
~ & $L_{class}{+}L_{CKD}$ & 95.67 & 97.43 & 97.07 & 83.17 & \textbf{96.46} & \textbf{63.42} & 88.87 & 47 \\ 
~ & $L_{class}{+}L_{CKD}$+SW & 95.5 & 97.06 & 97.2 & 82.92 & 95.63 & 62.85 & 88.53 & 43 \\ 
~ & CKD & \textbf{95.75} & \textbf{97.45} & \textbf{97.32} & \textbf{83.93} & 96.11 & 63.18 & \textbf{88.96} & \textbf{49} \\ 
\hline
\multirow{7}{*}{\rotatebox[origin=c]{90}{Verification}} & Face CE & 4.06 & 3.5 & 1.58 & 4.79 & 4.23 & 11.84 & 5 & 100 \\ 
~ & CE & 7.33 & 7.26 & 3.8 & 8.37 & 12.23 & 18.34 & 9.55 & 0 \\ 
~ & $L_{class}$+SW+SBS & 7.21 & 7.47 & 4.12 & 8.4 & 8.62 & 18 & 8.97 & 13 \\ 
~ & $L_{class}{+}L_{F2P}$+SW+SBS & 7.06 & 7.12 & 3.63 & 8.11 & 8.61 & 16.89 & 8.57 & 22 \\ 
~ & $L_{class}{+}L_{CKD}$ & 6.19 & 6.02 & 3.25 & 6.81 & 7.5 & 16.23 & 7.67 & 41 \\ 
~ & $L_{class}{+}L_{CKD}$+SW & 6.03 & 6.03 & \textbf{3.02} & 6.76 & 7.5 & 15.73 & 7.51 & 45 \\ 
~ & CKD & \textbf{5.61} & \textbf{5.48} & 3.13 & \textbf{6.53} & \textbf{6.77} & \textbf{15.16} & \textbf{7.11} & \textbf{54} \\ 
\hline
\end{tabular}
}
\label{table:model_ablation}
\end{table*}

Experiments are performed to evaluate the learned embedding of CKD for periocular identification or verification. This section describes the benchmarking datasets, the empirical settings and configurations, and the ablation study, followed by a thorough comparison and empirical analysis.

\subsection{Dataset Description}
\label{sec:exp_data}
Our experiments employ a summation of six datasets, including Ethnic \cite{tiong2019periocular}, PubFig \cite{kumar2009attribute}, FaceScrub \cite{ng2014data}, Imdb-Wiki \cite{IMDB}, AR \cite{martinez1998ar}, and YTF \cite{wolf2011face}. For training and validation purposes, a random subset is sampled from VGGFace \cite{parkhi2015deep} and Ethnic. The unseen testing set (composed of the gallery and the probe images) is configured in the open-world setting for a generally more challenging problem, where some identities are non-existent in the training set. 

\noindent
\textbf{Training Dataset} consists of 1054 subjects (577 subjects from the VGG face dataset and 477 subjects from the Ethnic dataset \cite{tiong2019periocular}) with 238,919 images as a whole for both face and periocular images. We crop the periocular region from each face image to yield the periocular images. Then, we randomly split all images into 166,737 for training, 47,372 for validation, and the rest for testing data.

\noindent
\textbf{Ethnic Dataset \cite{tiong2019periocular}} contains five ethnic groups: Asian, African, White, Middle Eastern, and Latin American. The subjects are sportsmen, celebrities, and politicians, and the images are acquired in unconstrained environments. There are 329 subjects in total (disjoint with the 477 subjects in the training dataset), 1,645 gallery images, five images for each subject, and 24,171 probe images.

\noindent
\textbf{PubFig Dataset \cite{kumar2009attribute}} comprises 200 subjects from real-life images collected from the internet. Images are acquired in the wild without any other settings. In our work, there are three probe sets. The gallery set has 9,221 images, and the probe sets have 6,138, 6,101, and 7,680 images, respectively.

\noindent
\textbf{FaceScrub Dataset \cite{ng2014data}} consists of 530 subjects where images are acquired in unconstrained environments. Images were acquired with various poses, illuminations, facial expressions, and backgrounds. This dataset has two probe sets. The gallery set has 31,066 images, the first probe set has 21,518 images, and the second has 27,292 images.

\noindent
\textbf{Imdb-Wiki Dataset \cite{IMDB}} collects 2,129 subjects with images taken in the wild environment and has three probe datasets. The gallery set has 40,241 images; the first probe has 17,658 images, the second one has 15,252 images, and the third one has 16,273 images in total.

\noindent
\textbf{AR Dataset \cite{martinez1998ar}} consists of faces with varying illumination, expression, blur, and occlusion conditions. The setup of the camera's focal distance and illumination conditions are controlled, but several distracting factors are added for each probe set. This dataset comprises 100 subjects, and periocular images are acquired by cropping the face images. The gallery set contains seven images per subject resulting in 700 images. The first probe set has blurred images with four extents of Gaussian filters applied to 2,800 images. The second one is conducted with various expressions, illumination, and blurring resulting in 1,400 images. The third probe set has zero-outed square occlusion of various sizes resulting in 3500 images. The last probe set occludes with a scarf resulting in 600 images. A detailed description of the AR dataset can be found in \cite{martinez1998ar}.  

\noindent
\textbf{YouTube Face (YTF) Dataset \cite{wolf2011face}} is an online YouTube video repository for 1,595 identities capturing a wide range of awful visual variations, including low-resolution and motion-blurred footage. The YTF evaluation protocol is originally set up for face verification. For periocular identification, we single out 225 subjects with at least four videos per subject in our experiments, where three videos are designated for the gallery setting and the remaining for the probe set. This results in 150,259 gallery images and 36,995 probe images.

\subsection{Experiment Settings}
\label{sec:exp_set}
We implement CKD with PyTorch \cite{paszke2019pytorch} trained on a single Nvidia GeForce RTX2080Ti GPU. The input dimensions for face and periocular images are 128×128 pixels and 48×128 pixels, respectively. The CKD backbone is ResNet-18 \cite{he2016deep} built with shared weights (SW) and batch-norm statistics (SBS), followed by projection heads for the face and periocular. The SW layers comprise three residual groups (indicated by Res1 to Res3), each containing two squeeze-and-excitation residual blocks \cite{hu2018squeeze}. On the other hand, each projection head is interleaved by a single SE residual block (Res4), an embedding layer, and a prediction layer used for supervised learning.

The entire training stage is optimized with stochastic gradient descent for 90 epochs in total, where the initial learning rate of 0.1 is decayed by a factor of 0.1 at epochs 30, 60, and 80. Other optimization hyperparameters include 0.9 for momentum and 5.0×10-4 for weight decay, and the temperature $\tau$ is set to 2.5. For weight initialization, we follow the default setting in PyTorch.

Other baseline models that involve single-stage training are trained with the same optimization setting. For two-stage baseline models such as KD, the face network is trained with 50 epochs and is distilled to train the periocular network with 40 epochs. For the two-stage models, the learning rate of the face network decays by a factor of 0.1 at epochs 15, 30, and 40, while the learning rate of the periocular network decays at epochs 10, 20, and 30 with the same decay rate.

For the identification task, we report the average cumulative match curve (CMC) \cite{hu2018squeeze} for rank-1 up to rank-10 identification rate (\%). CMC generally illustrates the retrieval rate of probe samples from the gallery images. 
In the case of multiple probe sets, the average result is reported by varying the choice of gallery set. For example, in the PubFig dataset consisting of 4 folds (one gallery set and three probe sets), we consider all combinations of gallery-probe pairs and average the 12 results.

We randomly select four samples per subject from the gallery set for verification. Each sample compares with every other one - resulting in 4x3x(number of identities)/2 positive cases and 4 x 4 x (number of identities) x (number of identities – 1)/2 negative cases in total. For example, in the case of PubFig, there are 200 subjects to extract 800 images. There would be 1,200 positive cases and 318,400 negative cases in total. We show the Equal Error Rate (EER) and the interpolated average ROC curve over all six datasets.

\begin{table*}
\centering
\caption{
{
Rank-1 identification rates and verification EER for six benchmarking datasets and their averages of various periocular recognition methods. Best performances are highlighted in bold font.
}
}
\resizebox{0.5\linewidth}{!}{   
\begin{tabular}{c l  c c c c c c c }
\hline
~ & ~ & Eth & PF & FS & IW & AR & YTF & Avg \\ 
\hline
\multirow{6}{*}{\rotatebox[origin=c]{90}{Identification}} & Attnet \cite{zhao2018improving} & 87.29 & 90.93 & 91.07 & 65.35 & 82.38 & 45.78 & 77.13 \\ 
~ & OC-LBCP \cite{tiong2019periocular} & 92.43 & 95.48 & 95.64 & 76.21 & 94.92 & 56.51 & 85.2 \\ 
~ & NIRPB \cite{hwang2020near} & 83.7 & 87.95 & 87.8 & 58.24 & 78.25 & 43 & 73.16 \\ 
~ & L2SR \cite{jung2020l2sr} & 94.02 & 96.42 & 96.76 & 80.34 & 94.3 & 59.2 & 86.84 \\ 
~ & GLSR \cite{jung2020glsr} & 94.24 & 96.62 & 96.96 & 80.82 & 94.67 & 59.77 & 87.18 \\ 
~ & CMBNet \cite{ng2022conditional} & 92.86 & 96.91 & 96.42 & 82.74 & 93.62 & - & - \\
~ & ASTNet \cite{borza2023adaptive} & 87.27  & 90.64  & 91.31  & 62.75  & 90.89  & 50.21  & 78.85 \\
~ & CKD & \textbf{95.75} & \textbf{97.45} & \textbf{97.32} & \textbf{83.93} & \textbf{96.11} & \textbf{63.18} & \textbf{88.96} \\ 
\midrule
\multirow{6}{*}{\rotatebox[origin=c]{90}{Verification}} & Attnet \cite{zhao2018improving} & 8.93 & 8.61 & 5.26 & 9.51 & 14.85 & 19.12 & 11.05 \\ 
~ & OC-LBCP \cite{tiong2019periocular} & 7.98 & 8.38 & 4.83 & 9.59 & 9.6 & 19.22 & 9.93 \\ 
~ & NIRPB \cite{hwang2020near} & 9.69 & 9.25 & 5.5 & 10.22 & 15.28 & 19.67 & 11.6 \\ 
~ & L2SR \cite{jung2020l2sr} & 6.76 & 6.32 & 3.54 & 7.57 & 9.98 & 17.51 & 8.61 \\ 
~ & GLSR \cite{jung2020glsr} & 6.54 & 6.39 & 3.2 & 7.34 & 9.24 & 18.06 & 8.46 \\ 
~ & CMBNet \cite{ng2022conditional} & \textbf{4.56} & 6.78 & 3.98 & 7.64 & 8.27 & - & - \\
~ & ASTNet \cite{borza2023adaptive} & 10.60  & 12.38  & 6.98  & 13.21  & 12.73  & 19.33  & 12.54 \\
~ & CKD & 5.61 & \textbf{5.48} & \textbf{3.13} & \textbf{6.53} & \textbf{6.77} & \textbf{15.16} & \textbf{7.11} \\ 
\hline
\end{tabular}
}
\label{table:compare_periocular_sota}
\end{table*}

\subsection{Ablation Study}
\label{sec:exp_ablation}
The CKD model integrates two key components: the CKD loss transfers knowledge at the prediction layer and consistency constraint at the feature layers. The CKD loss encapsulates the face-to-periocular KL loss and periocular-to-face KL loss. The knowledge transfer is effected via shared weights and shared batch statistics. To analyze the contribution, we consider the following configurations:

``\textbf{Face CE}'' is trained on face images with the standard CE loss and serves as an upper bound.
``\textbf{CE}'' is trained on periocular images with the standard CE loss and serves as a lower bound.
``\textbf{$L_{class}$+SW+SBS'}' excludes the CKD loss and trains only with the classification loss for face and periocular images.
``\textbf{$L_{class} {+} L_{F2P}$+SW+SBS}'' replaces the bi-directional KL divergence in CKD with a one-way distillation loss (without reverse KL). Here, $L_{F2P}{=}D_{KL}(\mathbf{p}^F \parallel \mathbf{p})$ is the face-to-periocular distillation loss.
``$L_{class}{+}L_{CKD}$'' is the CKD model without knowledge distillation through the feature layers.
``$L_{class}{+}L_{CKD}$+SW'' is the CKD model without sharing batch norm statistics.
All models here are trained with 90 epochs.


To quantify the contribution to the performance improvement of periocular recognition, we measure a relative `Gain' from the average performance $avg_P$ of the periocular CE baseline,
$Gain = \frac{avg_C-avg_P}{avg_F-avg_P} \times 100 \% $,
with its upper bounded set by the average performance $avg_F$ of the Face CE network. Here, $avg_C$ is the average performance of the configuration model.

As summarized in Table \ref{table:model_ablation}, the proposed CKD model improves the average rank-1 identification rate by 3.47\% and reduces the average EER by 2.44\% across all benchmark datasets by 3.47\%. The relative gains are 49\% for identification and 54\% for verification, which is significantly greater than the vanilla KD.

The result of ``$L_{class}$+SW+SBS'' in Table \ref{table:model_ablation} shows that CKD loss is necessary for periocular recognition enhancement. Moreover, in the single-stage KD regiment of CKD, minimization of reverse KL divergence is crucial, as shown by the result of $L_{class}{+}L_{F2P}$ one-way knowledge transfer from face to periocular reduces the relative gains by 39\% in the identification rate and 33\% in EER. This shows that practicing consistency in knowledge distillation demands periocular recognition enhancement.

The results of ``$L_{class}+L_{CKD}$'' and ``$L_{class}+L_{CKD}$+SW'' in Table \ref{table:model_ablation} exhibit the importance of CKD via the network feature layers by SW and SBS. Removing any knowledge transfer through the feature layers (i.e., removing both SW and SBS) reduces the average EER by 13\%, indicating that CKD at feature layers makes the periocular network more robust. Moreover, excluding the SBS in the feature layers deteriorates the relative gains by 7\% in the identification rate and 9\% in EER, validating the significant contribution of SBS.

Overall, imposing consistency in the knowledge distillation through prediction at feature layers induces relative gains by 31\% in the average identification rate and by 27\% in the average EER compared to the standard KD.

\begin{table*}[t]
\caption{
{
Rank-1 identification rates and verification EER for six benchmarking datasets and their averages of various KD models. Both periocular to periocular KD and face to periocular KD models are compared. The highest accuracy is highlighted in bold font.
}
}
\centering
\resizebox{0.9\linewidth}{!}{   
\begin{tabular}{c l|c c c c c c c c |c c c c c c c c}
\toprule
& \multirow{2}{*}{Method} & \multicolumn{8}{c|}{Identification} & \multicolumn{8}{c}{Verification} \\ 
\cmidrule{3-18}
& & Eth & PF & FS & IW & AR & YTF & Avg & Gain & Eth & PF & FS & IW & AR & YTF & Avg. & Gain \\ 
\cmidrule{1-18}
& Face CE (Res18) & 97.48 & 99.19 & 98.66 & 90.18 & 91.74 & 78.52 & 92.63 & 100 & 4.06 & 3.5 & 1.58 & 4.79 & 4.23 & 11.84 & 5 & 100 \\ 
& Periocular CE & 92.82 & 95.78 & 96.44 & 77.63 & 93.5 & 56.78 & 85.49 & 0 & 7.33 & 7.26 & 3.8 & 8.37 & 12.23 & 18.34 & 9.55 & 0 \\ 
\cmidrule{1-18}
\multirow{8}{*}{\rotatebox[origin=c]{90}{Periocular-Periocular Res18}}
& AT \cite{zagoruyko2016paying} & 93.68 & 96.24 & 96.41 & 79.03 & 94.24 & 58.87 & 86.41 & 13 & 7.36 & 7.18 & 3.94 & 7.99 & 11 & 17.22 & 9.12 & 10 \\
& KD \cite{hinton2015distilling} & 94.12 & 96.55 & 96.94 & 80.5 & 94.38 & 57.97 & 86.74 & 18 & 6.84 & 6.32 & 3.18 & 7.47 & 9.11 & 17.01 & 8.32 & 27 \\ 
& PKT \cite{passalis2018learning} & 94.11 & 96.49 & 96.76 & 80.36 & 94.42 & 59.13 & 86.88 & 19 & 6.88 & 7 & 3.33 & 7.57 & 10.24 & 17.5 & 8.75 & 18 \\ 
& SPKD \cite{tung2019similarity} & 94.21 & 96.42 & 96.79 & 80.53 & 94.85 & 58.9 & 86.95 & 20 & 7.11 & 6.58 & 3.4 & 7.79 & 10.63 & 17.6 & 8.85 & 15 \\
& RKD \cite{park2019relational} & 94.45 & 96.74 & 96.79 & 81.05 & 95.29 & 59.1 & 87.23 & 24 & 6.54 & 6.18 & 3.37 & 7.55 & 10.39 & 17.33 & 8.56 & 22 \\ 
& DDGSD \cite{xu2019data} & 95.15 & 97 & 96.68 & 82.1 & 94.58 & 60.49 & 87.67 & 30 & 6.54 & 6.45 & 3.52 & 6.92 & 8.9 & 17.77 & 8.35 & 27 \\ 
& BYOT \cite{zhang2019your} & 93.3 & 96.02 & 96.64 & 78.69 & 93.41 & 57 & 85.84 & 5 & 7.22 & 7.5 & 3.3 & 8.25 & 10.38 & 17.65 & 9.05 & 11 \\ 
& CWSD \cite{yun2020regularizing} & 92.78 & 95.89 & 97 & 78.56 & 93.1 & 56.55 & 85.65 & 2 & 7.79 & 6.74 & 2.9 & 7.78 & 9.73 & 18.77 & 8.95 & 13 \\ 
\cmidrule{1-18}

\multirow{6}{*}{\rotatebox[origin=c]{90}{Face-Periocular Res18}}
& AT \cite{zagoruyko2016paying} & 93.65 & 96.24 & 96.21 & 78.67 & 94.75 & 58.88 & 86.4 & 13 & 7.1 & 7.52 & 3.99 & 8.3 & 9.73 & 18 & 9.11 & 10 \\ 
& KD \cite{hinton2015distilling} & 93.92 & 96.4 & 96.69 & 79.47 & 95.13 & 59.07 & 86.78 & 18 & 6.76 & 6.95 & 3.75 & 8.18 & 9.13 & 18.16 & 8.82 & 16 \\ 
& PKT \cite{passalis2018learning} & 92.45 & 95.13 & 95.3 & 74.6 & 93.74 & 57.51 & 84.79 & -10 & 8.24 & 8.25 & 4.46 & 8.87 & 9.12 & 17.72 & 9.44 & 2 \\ 
& SPKD \cite{tung2019similarity} & 93.47 & 96.17 & 96.08 & 78.68 & 94.44 & 57.53 & 86.06 & 8 & 7.41 & 7.89 & 4.17 & 8.37 & 9.47 & 17.71 & 9.17 & 8 \\
& RKD \cite{park2019relational} & 93.25 & 96.10 & 96.23 & 78.56 & 94.45 & 58.09 & 86.11 & 9 & 7.03 & 7.56 & 3.92 & 8.45 & 8.25 & 17.62 & 8.80 & 16 \\ 
& ML \cite{zhang2018deep} & 94.73 & 96.83 & 96.90 & 81.84 & 95.52 & 60.95 & 87.80 & 32 & 7.03 & 6.19 & 3.54 & 7.30 & 8.98 & 16.77 & 8.3 & 27 \\

\cmidrule{1-18}
\multirow{5}{*}{\rotatebox[origin=c]{90}{Face-Periocular Res50}}
& Face CE (Res50) & 97.91  & 99.36  & 98.79  & 91.53  & 92.78  & 80.55  & 93.49 & - & 4.00  & 3.00  & 1.49  & 4.50  & 4.50  & 11.94  & 4.91 & -  \\ 
\cmidrule{2-18}
& AT \cite{zagoruyko2016paying} & 93.73  & 96.14  & 96.29  & 79.06  & 94.85  & 59.63  & 86.62 & -  & 7.03  & 7.39  & 4.10  & 8.33  & 8.62  & 17.11  & 8.76 & - \\ 
& KD \cite{hinton2015distilling} & 93.71 & 96.29 & 96.61 & 79.21 & 94.70 & 58.69 & 86.54 & - & 7.07 & 6.76 & 3.44 &  7.99 & 9.14 & 17.93 & 8.72 & -  \\ 
& PKT \cite{passalis2018learning} & 92.09  & 95.02  & 95.60  & 74.35  & 93.38  & 57.75  & 84.70 & -  & 8.20  & 8.06  & 4.27  & 9.22  & 9.21  & 17.23  & 9.37 & -  \\ 
& SPKD \cite{tung2019similarity} &93.58  & 96.13  & 96.17  & 78.59  & 94.34  & 57.97  & 86.13 & -   & 6.69  & 7.44  & 4.13  & 8.34  & 10.78  & 17.48  & 9.14 & -  \\

\cmidrule{1-18}

& CKD (ours)& \textbf{95.75} & \textbf{97.45} & \textbf{97.32} & \textbf{83.93} & \textbf{96.11} & \textbf{63.18} & \textbf{88.96} & \textbf{49} & \textbf{5.61} & \textbf{5.48} & \textbf{3.13} & \textbf{6.53} & \textbf{6.77} & \textbf{15.16} & \textbf{7.11} & \textbf{54} \\ 
\bottomrule
\end{tabular}
}
\label{table:compare_kd}
\end{table*}

\begin{figure*}[t]
\centering
\includegraphics[width=0.99\linewidth]{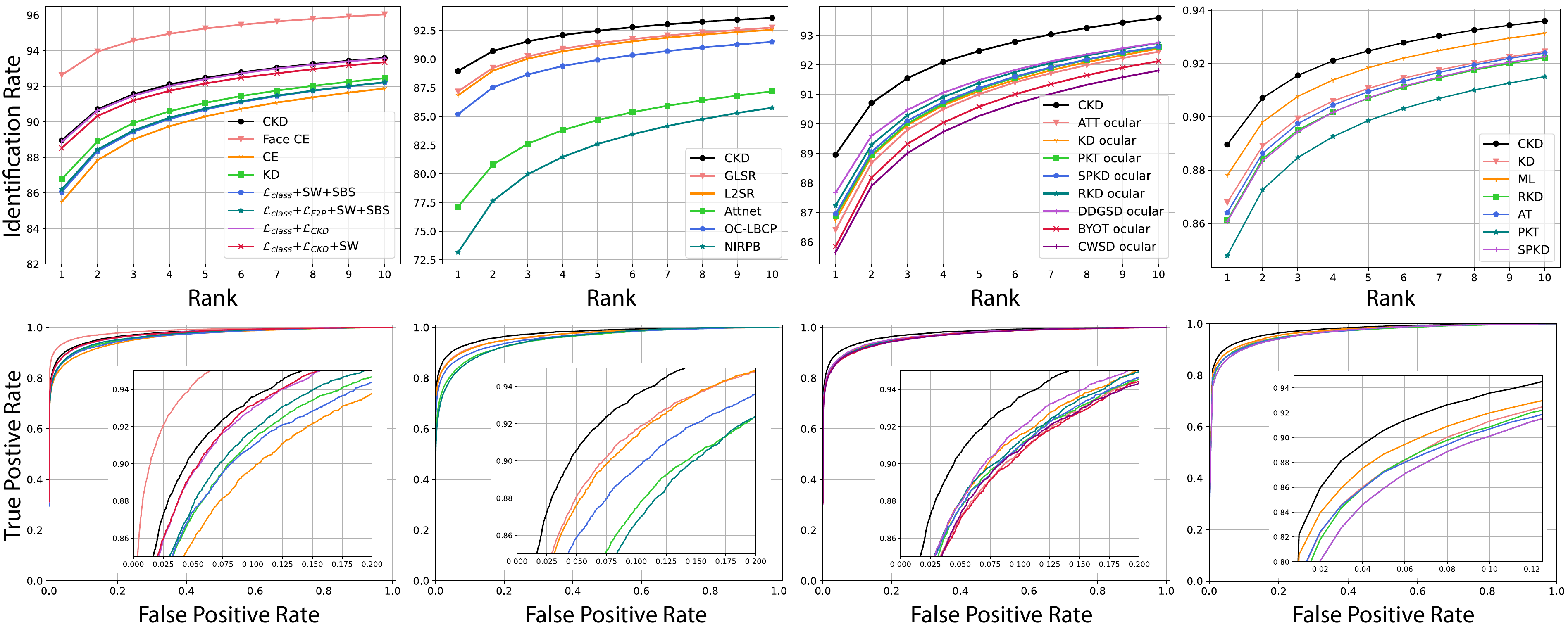}
\caption{
The performance comparison for periocular identification (the first row) and verification (the second row) was measured by CMC and indicated by the ROC curve, respectively. The black line is of our proposed model CKD. (First column) The results of the ablation study are in Sec.~\ref{sec:exp_ablation}. (Second column) Experimental comparison with state-of-the-art periocular recognition methods in Sec.~\ref{sec:exp_comp_peri}. (Third column) Experimental comparison with different KD methods enhances the periocular network by self-distillation from ocular images (Sec.~\ref{sec:exp_comp_kd}). (Fourth column) Experimental comparison with different KD methods enhances the periocular network by knowledge distillation from face images (Sec.~\ref{sec:exp_comp_kd}).
}
\label{fig:results}
\end{figure*}

\subsection{Experimental Results for Comparison}

\subsubsection{Comparison to Periocular Recognition State-of-the-Art Models}
\label{sec:exp_comp_peri}

This section compares CKD with other state-of-the-art (SOTA) targeting periocular identification or verification under the same evaluation protocol, including Attnet \cite{zhao2018improving}, OC-LBCP \cite{tiong2019periocular}, NIRPB \cite{hwang2020near}, L2SR \cite{jung2020l2sr}, and GLSR \cite{jung2020glsr}, CMBNet \cite{ng2022conditional}, and ASTNet \cite{borza2023adaptive}. Similar to CKD, we re-implemented these deep learning-based SOTA with ResNet-18 based on our training dataset.

Overall, Attnet \cite{zhao2018improving}, OC-LBCP \cite{tiong2019periocular}, NIRPB \cite{hwang2020near}, and L2SR \cite{jung2020l2sr} are only trained with periocular images which show its limited performance from restricted information to use. Attnet \cite{zhao2018improving} implements a pretrained attention module focusing on the eyebrow and ocular region trained on a distance-driven cross-entropy loss. However, experiments were done in a controlled environment and heavily relied on the performance of the attention module resulting in lower performance than CKD. NIRPB \cite{hwang2020near}, and OC-LBCP \cite{tiong2019periocular} use additional information such as feature descriptors or near-infrared to complement the insufficient periocular information, but using the many faces information shown with CKD shows better performance. Finally, GLSR \cite{jung2020glsr} uses face images to supplement information for periocular but is highly dependent on pretrained models. CMBNet \cite{ng2022conditional} jointly trains face and periocular in a conditional manner equipped with contrastive loss. ASTNet \cite{borza2023adaptive} adds a local branch that attends to the informative area of the periocular region using learned transformation functions. Our proposed method learns and distills the relational knowledge between face and periocular adaptively and shows better results. Table \ref{table:compare_periocular_sota} shows that CKD outperforms other periocular SOTA models by a remarkable margin.

\subsubsection{Comparison to Other Knowledge Distillation Models}
\label{sec:exp_comp_kd}

In this section, we compare CKD to other prevailing KD models (\cite{hinton2015distilling}, \cite{zagoruyko2016paying}, \cite{tung2019similarity}, \cite{park2019relational}, \cite{passalis2018learning}, \cite{zhang2019your}, \cite{xu2019data}, \cite{yun2020regularizing}) designated neither for periocular nor other biometrics, and we summarize these models in Table \ref{table:compare_kd} and Fig.~\ref{fig:results}. We first learn the periocular-to-periocular distillation from a large cumbersome network to a compressed network (training KD for network compression \cite{hinton2015distilling}, \cite{zagoruyko2016paying}, \cite{tung2019similarity}, \cite{park2019relational}, \cite{passalis2018learning}), or within itself (also referred to as self-distillation \cite{zhang2018deep}, \cite{xu2019data}, \cite{yun2020regularizing}). After that, the face-to-periocular KD methods are analyzed. As most of the KD models exercise for only a single data type, we twist them into the face to the periocular mode in our experiments. This demonstrates that the naïve execution of KD methods for face-to-periocular distillation does not perform well.

\noindent
\textbf{Periocular-to-Periocular Distillation} For exploration, ResNet-50 serves as the teacher model trained with the same setting as the vanilla periocular model trained only with cross-entropy loss. The student network, on the other hand, recruits the ResNe-18 backbone trained according to Sec.~\ref{sec:exp_set}. For the self-distillation models \cite{zhang2019your}, \cite{xu2019data}, \cite{yun2020regularizing}, only ResNet-18 is considered.

We train eight different KD models to compare in total. According to Table \ref{table:compare_kd}, all KD models show improvement in performance in varying degrees compared with the CE baseline. As mentioned in previous sections, \cite{hinton2015distilling} generates soft target prediction from a larger pre-trained fixed model and transfers the knowledge, but using a fixed model and doesn't consider the relational information between face and periocular and shows its limit. Likewise, \cite{zagoruyko2016paying} merely distills the fixed teacher network's attention map of intermediate features to the student network to induce similar attention patterns. \cite{tung2019similarity} \cite{park2019relational} \cite{tung2019similarity} distills the relational statistics of intermediate feature activations computed between different samples within the batch using different metrics but shows lower performance compared with CKD. Recent self-distillation models \cite{zhang2019your}, \cite{xu2019data}, \cite{yun2020regularizing} makes marginal gains, as shown in the table. Still, CKD performs better by fully utilizing the information of the face, which contains more salient features.

\begin{figure}[t]
\centering
\includegraphics[width=0.995\linewidth]{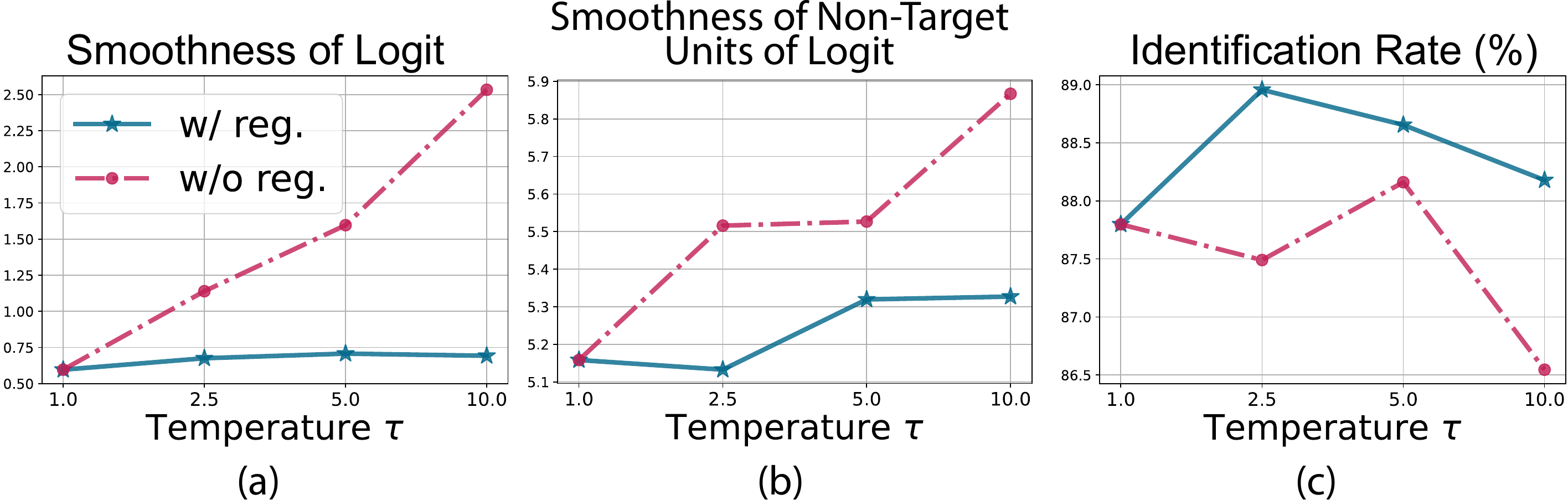}
\caption{
Here, the model with the regularizer is our full model CKD.
(a) Smoothness of the face network logit $\mathbf{z}^F$ measured by the entropy of its corresponding posterior (softmax output). (b) Smoothness of the non-target components $z^F_k$ ($k \neq y$) of the logit $\mathbf{z}^F$ measured by the entropy of softmax output of $(z^F_k)_{k \neq y}$ (c) Corresponding identification rates (measured in CMC rank-1) in average for the benchmark datasets. The results indicate two aspects of the sparsity-oriented regularizer $R(\mathbf{z}^F)$: First, as shown in (a), over-smoothing by increasing the temperature $\tau$ is prevented by the regularizer. Secondly, for a fixed temperature in (b), the inclusion of the regularizer increases the sparsity of non-target components of the logit, indicating that the regularizer stores inter-class relationship information in the non-target components of the prediction output $p^F_k$ for $k\neq y$ by making them sparse.
}
\label{fig:ablate_reg}
\end{figure}

\noindent
\textbf{Face-to-Periocular Distillation} 
{
We examine in this section face-to-periocular distillation. We use ResNet-18 and ResNet-50 for the teacher to compare both teachers with the same and larger capacity. The teacher network of \cite{hinton2015distilling}, \cite{tung2019similarity}, and \cite{passalis2018learning} is realized by a pre-trained face model. However, the spatial dimension for the face attention map of \cite{zagoruyko2016paying} has to be adjusted since its size differs from the periocular attention map. We, therefore, downsample them with adaptive average pooling to have the same size as periocular. In this section, we want to show the novelty of how CKD distills from face to periocular and how the resulting model differs from the other KD variants.
}
Our revelation in Table 3 shows that all comparing KD models improve the base performance, except PKT \cite{passalis2018learning}. Given the gains, the performance gain is relatively marginal compared with the periocular to periocular distillation discussed in the preceding section. This suggests these methods are restricted to only the same data type distillation except for ML. ML distills well but does not correctly capture facial image discriminative relationship information. RKD, on the other hand, extracts relationship information only locally from the given mini-batch. We find that local information transfer by RKD from face to periocular is weak as it cannot handle data discrepancy between face and periocular. In contrast, the proposed CKD can extract profound global inter-class relationships from face images and effectively transfer them to the periocular network, demonstrating its superiority.

\subsubsection{Verification on General Face Protocol}
\label{sec:exp_general}
We conduct validation experiment on LFW\cite{huang2008labeled}, CALFW, CFP-FF\cite{sengupta2016frontal}, and AgeDB\cite{moschoglou2017agedb} which are widely used face verification dataset. Since there aren't periocular images available for these datasets, we crop the aligned images and sanitize by excluding poorly cropped images which are caused by invisible periocular due to excessive pose change and ill periocular alignment where such cases are irrelevant with our work using aligned periocular images. This results in 5,749 pairs for LFW, 5,569 pairs for CALFW, 6,564 pairs for CFP-FF, and 5,575 pairs for AgeDB. As shown in Tab.~\ref{table:general_periocular_eer}~\ref{table:general_periocular_auroc}, CKD significantly improves both AUROC and EER comparing to the model only trained with Cross Entropy loss often with remarkable margins. 

\begin{table}[t]
\centering
\caption{Verification AUROC results on LFW, CALFW, CFP-FF, and AgeDB datasets.}
\resizebox{0.75\linewidth}{!}{   
    \begin{tabular}{c l c c c c}
    \hline
    ~ & Method & LFW & CALFW & CFP-FF & AgeDB \\ 
    \hline
    \multirow{2}{*}{\rotatebox[origin=c]{90}{AUROC}}  & Periocular CE & 96.22 & 90.83 & 98.41 & 88.37 \\ 
    & Periocular CKD & \textbf{97.23} & \textbf{92.44} & \textbf{98.77} & \textbf{90.98} \\ 
    \hline
    \end{tabular}
}
\label{table:general_periocular_eer}

\end{table}

\begin{table}[t]
\centering
\caption{Verification EER results on LFW, CALFW, CFP-FF, and AgeDB datasets.}
\resizebox{0.75\linewidth}{!}{   
    \begin{tabular}{c l c c c c}
    \hline
    ~ & Method & LFW & CALFW & CFP-FF & AgeDB \\ 
    \hline
    \multirow{2}{*}{\rotatebox[origin=c]{90}{EER}} & Periocular CE & 9.51 & 17.2 & 5.64 & 19.19 \\ 
    & Periocular CKD & \textbf{7.69} & \textbf{14.54} & \textbf{4.6} & \textbf{16.59} \\ 
    \hline
    \end{tabular}
}
\label{table:general_periocular_auroc}
\end{table}

\subsubsection{CKD and Low Resolution Face Images}
\label{sec:exp_low_res}
In this section, we replace the periocular images and module to low face and apply CKD in order to show that our method is not only limited to improving periocular recognition performance. For the resolutions, we use $8\times8$ and $16\times16$ which are obtained through downsampling the high resolution face images, which are $128\times128$, and conseqeuntly umsampling back to $128\times128$. Besides the resolution of the input, the rest of the setup are identical as periocular experiments. We compare the identification and verification under the same protocol we have used in the previous experiments for each datasets. For both $8\times8$ and $16\times16$, CKD shows remarkable improvements as shown in Tab.~\ref{table:low_resolution}.

\begin{table}[t]
\centering
\caption{Identification and verificaiton on $8\times8$ and $16\times16$ resolution face recognition.
}
\resizebox{0.995\linewidth}{!}{   
\begin{tabular}{c l c c c c c c c}
\hline
~ & Resolution/Method & Eth & PF & FS & IW & AR & YTF & Avg \\ 
\hline
\multirow{4}{*}{\rotatebox[origin=c]{90}{Identification}} & $128 \times 128$ CE & 97.48  & 99.19  & 98.66  & 90.18  & 91.74  & 78.52  & 92.63  \\
& $ 16 \times 16$ CE & 95.56  & 97.70  & 97.61  & 82.60  & 91.74  & 75.54 & 90.12 \\
& $16 \times 16$ CKD & \textbf{96.42 } & \textbf{98.17 } & \textbf{97.61 } & \textbf{84.62 } & \textbf{93.68 } & \textbf{78.846} & \textbf{91.56} \\
& $8 \times 8$ CE & 90.75  & 93.97  & 94.64  & 69.23  & 84.64  & 65.32 & 83.09  \\
& $8\times8$ CKD  & \textbf{93.88 } & \textbf{96.05 } & \textbf{95.47 } & \textbf{75.38 } & \textbf{89.53 } & \textbf{72.81} & \textbf{87.18} \\ 
\hline
\multirow{4}{*}{\rotatebox[origin=c]{90}{Verification}}  & $128 \times 128$ CE & 4.06  & 3.50  & 1.58  & 4.79  & 4.23  & 11.84  & 5.00 \\
& $16 \times 16$ CE & 22.80  & 30.64  & 29.15  & 30.38  & 23.61  & 24.83  & 26.90 \\
& $16\times16$ CKD  & \textbf{15.05 } & \textbf{22.95 } & \textbf{20.68 } & \textbf{23.99 } & \textbf{14.31 } & \textbf{18.73 } & \textbf{19.28} \\ 
& $8 \times 8$ CE & 28.57  & 34.81  & 33.44  & 34.43  & 23.98  & 28.78  & 30.67 \\
& $8 \times 8$ CKD & \textbf{25.08 } & \textbf{32.75 } & \textbf{30.87 } & \textbf{31.80 } & \textbf{22.87 } & \textbf{24.89 } & \textbf{28.04} \\
\hline
\end{tabular}
}
\label{table:low_resolution}
\end{table}

\subsection{Empirical Analysis of CKD}

We conduct a detailed analysis of our model CKD in this part of the experiments. We first analyze the novel sparsity-oriented regularizer that has been theoretically observed in the equivalent form of the CKD total objective loss (Sec.~\ref{sec:exp_anal_reg}). Then, in Sec.~\ref{sec:exp_anal_cons}, we analyze the impact of consistency in the process of knowledge distillation, observing how it affects the embedding structures, effectiveness of knowledge transfer, prediction calibration, and the cluster quality of embeddings.


\subsubsection{Analysis of the Sparsity-Oriented Regularizer}
\label{sec:exp_anal_reg}

\noindent
\textbf{Settings.} To analyze the impact of the sparsity-oriented regularizer hidden in the CKD model loss, we train two models, our CKD as described in Sec.~\ref{sec:exp_set} and the model trained without the regularizer. The latter model is trained only with the label-smoothing loss $H(\widetilde{\mathbf{y}}, \mathbf{p}^F) + H(\widetilde{\mathbf{y}}^F, \mathbf{p})$ described in Eq.~\eqref{eq:loss_equiv}.

\noindent
\textbf{Analysis.} 
The comparison between these two models is given in the result of Fig.~\ref{fig:ablate_reg}. Here, we measure the smoothness of the face network prediction logit $\mathbf{z}^F$ by the entropy of its softmax output $H(\mathbf{p}^F)$, while the smoothness of the non-target components of the logit is measured by the entropy of the softmax output of the non-target components; namely, $H(\mathbf{q}^F)$ where $\mathbf{q}^F$ is the softmax output of $\mathbf{z}^F_{[1:K] \setminus \{y\}} = (z^F_1, \dots, z^F_{y-1}, z^F_{y+1}, \dots, z^F_K) \in \mathbb{R}^{K-1}$. 

The result shown in Fig.~\ref{fig:ablate_reg}a indicates that the regularizer prevents over-smoothing of the prediction even at high temperature $\tau$. (The regularizer has no effect at $\tau {=} 1$ since it is a fixed constant.)
On the other hand, the graph in Fig.~\ref{fig:ablate_reg}b shows that the prediction posterior involves more sparse behavior for the non-target classes $k {\neq} y$ as illustrated in Figs.~\ref{fig:effect_reg} and \ref{fig:comp_methods}. In other words, the non-target posterior is large for only specific identities $k{\neq}y$ when the regularizer is in effect. In contrast, the absence of the regularizer makes all non-target posteriors similar and uniform. It indicates that the regularizer extracts inter-class relationship information from face images and stores it in the non-target posterior.

\begin{figure}
\centering
\includegraphics[width=0.85\linewidth]{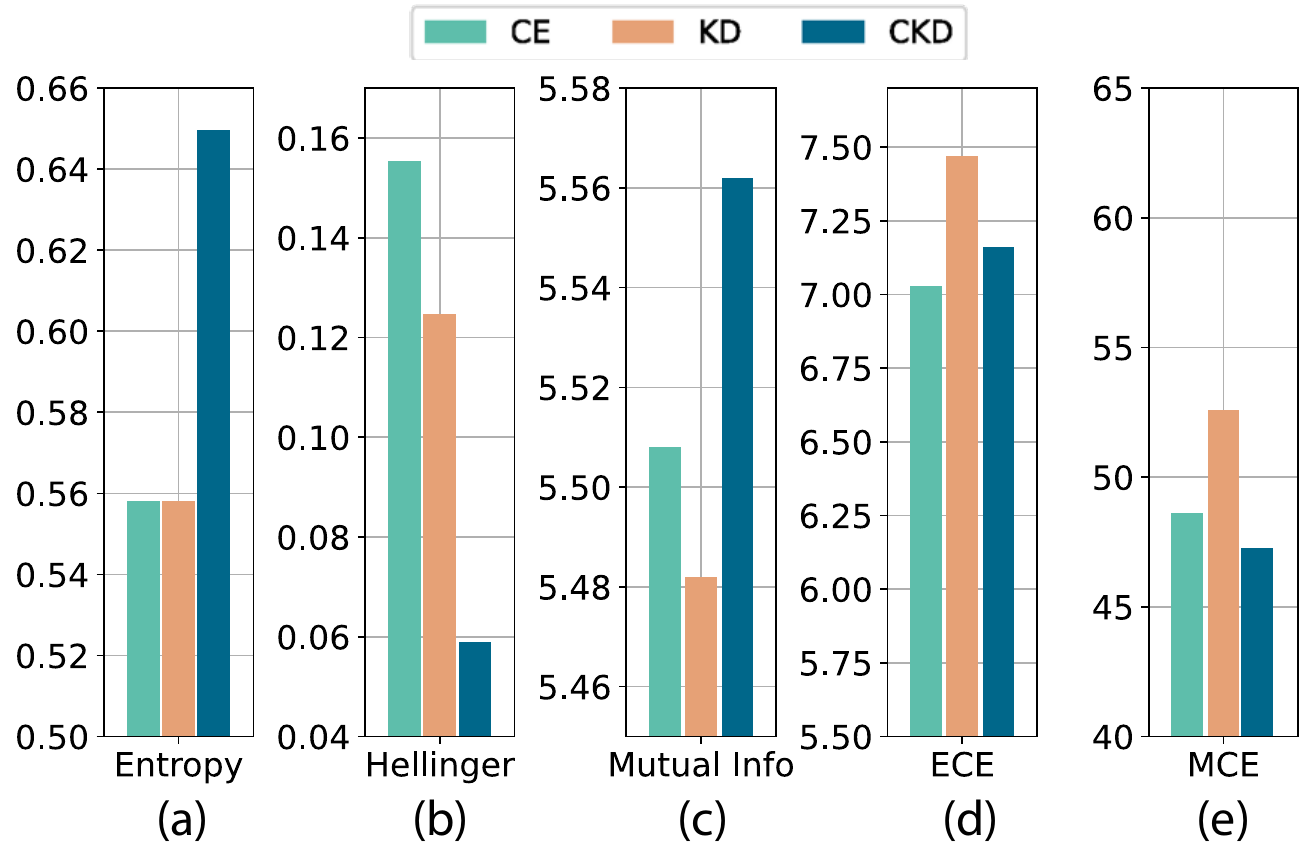}
\caption{
(a) Entropy $H(\mathbf{p}^F) = - \sum_{k=1}^K p^F_k \log p^F_k$ of the posterior of face network averaged over validation samples. Applying KD does not increase entropy, while CKD does. (b) Hellinger distance between the face network posterior $\mathbf{p}^F$ and the periocular network posterior $\mathbf{p}$, showing that $\mathbf{p}$ is very close to $\mathbf{p}^F$ in CKD compared to CE and KD baselines. The distance is averaged over validation samples.
(b) The mutual information between $\mathbf{p}$ and $\mathbf{p}^F$ averaged over validation samples, indicating high mutual information between face and periocular posteriors.
(d, e) Expected Calibration Error (ECE) and Maximum Calibration Error (MCE) quantify the discrepancy between prediction uncertainty and error rate. The CE baseline and CKD are on par in ECE, while the standard KD largely deteriorates the metric. In MCE, CKD reduced the calibration error compared to the CE and KD baselines.
}
\label{fig:stat}
\end{figure}

\begin{figure}[t]
\centering
\includegraphics[width=0.9\linewidth]{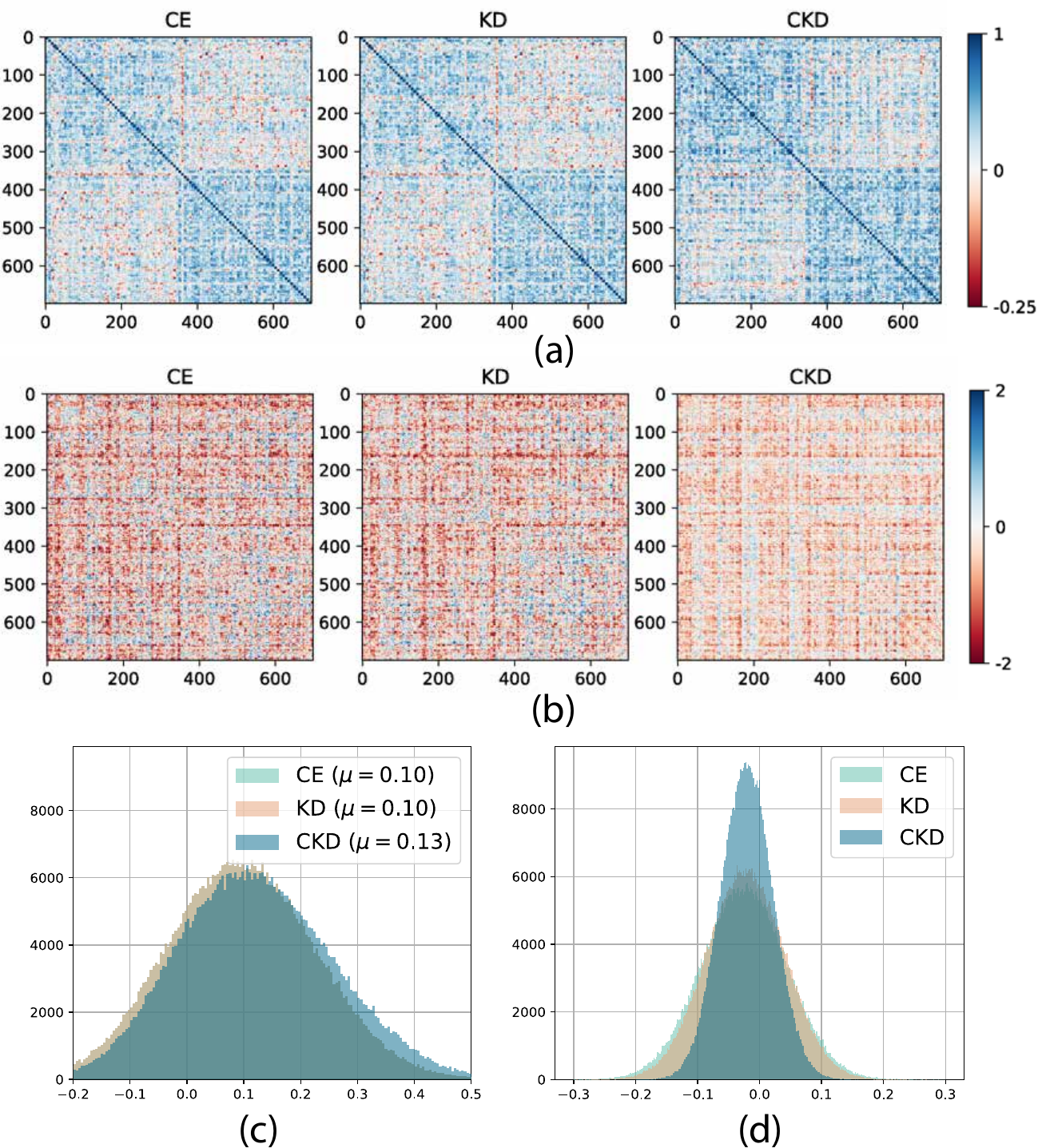}
\caption{
(a) The gram matrices of face prototype embeddings from CE, KD, and CKD. The $(i,j)$-th element of a gram matrix quantifies the cosine similarity between $i$-th class prototype embedding and $j$-th class prototype embedding. The gram matrix of CKD is denser (non-sparse) and contains larger elements, capturing a more meaningful identity-wise relational structure than the CE and KD. The CE baseline and CKD are on par with ECECE
(b) The difference between the gram matrix of face prototype embedding and periocular prototype embedding, where $(i,j)$-th element corresponds to the classes $i$ and $j$. 
(c) The distributions of cosine similarities in the gram matrices of (a), where the mean of CKD is the largest. 
(d) The distributions of the elements in the different matrices of (b). 
(b) and (d) shows that the difference is significantly smaller in CKD. In other words, the structural similarity between face and periocular embeddings is higher in CKD than in CD and KD baselines.
}
\label{fig:structure}
\end{figure}

\subsubsection{Analysis of Consistency in Distillation of Knowledge}
\label{sec:exp_anal_cons}

\noindent
\textbf{Settings.} To analyze the impact of consistency in knowledge distillation, we train three models as described in Sec.~\ref{sec:exp_set}: our proposed model CKD, a vanilla KD model, and the CE baseline trained only over periocular images. All models are trained on Training Dataset described in Sec.~\ref{sec:exp_data} and evaluated over corresponding validation and test datasets.

\subsubsection*{(a) Consistency in CKD enables the face network to attain profound inter-class relationship information of subject identities from face images.}
\label{sec:theory_entropy}
Suggested by the theoretical analysis in Sec.~\ref{sec:theory_entropy} with Corollary \ref{thm:entropy}, CKD acts as a learned-label smoothing mechanism with the sparsity-oriented regularizer, extracting profound inter-class relationship information from face images.
This is empirically evidenced by measuring the entropy of class posterior from the face network averaged per sample in Fig.~\ref{fig:stat}a. Moreover, the gram matrix of embeddings given in Fig.~\ref{fig:structure}a exhibits relatively less sparsity for the face embeddings of CKD, and the distribution of the embedding similarities in CKD has a larger mean than the baseline models (Fig.~\ref{fig:structure}c). The observations indicate that the face network of CKD learns more meaningful identity-wise relational information in the embedding space of face images than the CE and KD baselines.

\subsubsection*{(b) Consistency in CKD enables the effective transfer of knowledge from the face to the periocular network.}
As shown in Fig.~\ref{fig:loss_dbi}, imposing consistency between face and periocular networks in CKD makes the model minimize the face-to-periocular distillation loss $D_{KL}(\mathbf{p}^F \parallel \mathbf{p})$ much more effectively than the KD counterpart. As a result, CKD attains much higher statistical similarity and mutual information between face and periocular class prediction posteriors (Fig.~\ref{fig:stat}bc). Moreover, the transfer of consistent knowledge increases the structural similarity between face and periocular embeddings, as indicated by the difference between the gram matrix of face embeddings and that of periocular embeddings (Fig.~\ref{fig:structure}b). The corresponding distribution of the difference between the embedding similarity of the face and that of the periocular exhibits the same trend (Fig.~\ref{fig:structure}d). Overall, the profound identity-wise relational information attained from face images is effectively transferred to the periocular network via learning consistency between face and periocular.

\subsubsection*{(c) Consistency in CKD induces effective calibration of periocular network prediction, well-clustered periocular embeddings, and robust periocular recognition in the wild.}
By learning profound informative relational information of subject identities and regularizing the periocular prediction by the label-smoothing mechanism of CKD, the periocular network of CKD can attain effective prediction calibration over Training Dataset identities, as shown by Fig.~\ref{fig:stat}de. We note that Expected Calibration Error (ECE) and Maximum Calibration Error (MCE), measured in Fig.~\ref{fig:stat}d and Fig.~\ref{fig:stat}e, respectively, quantify the discrepancy between the uncertainty of prediction and the corresponding prediction error rate \cite{guo2017calibration}.

Along with the effective calibration of periocular predictions, effective transfer of the face knowledge involving profound identity-wise relation and label-smoothing mechanism of CKD make the periocular network resilient to overfitting (Fig.~\ref{fig:loss_dbi}b).
Moreover, transferring consistent knowledge of the face forces the periocular network to obtain well-clustered embeddings of periocular images (Fig.~\ref{fig:loss_dbi}c). Overall, the periocular embeddings of CKD exhibit robust periocular identification and verification in wild environments.

\begin{figure}[t]
\centering
\includegraphics[width=0.995\linewidth]{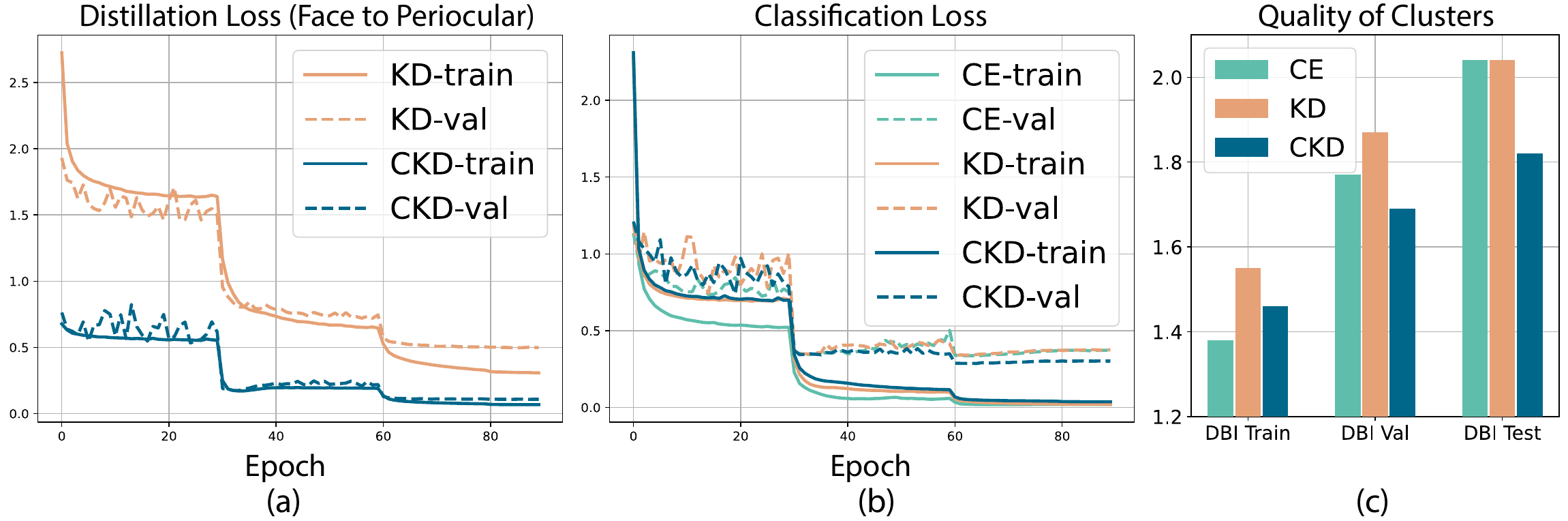}
\caption{
(a) The face-to-periocular distillation loss $D_{KL}(\mathbf{p}^F \parallel \mathbf{p})$. Although KD minimizes this distillation loss, it \textit{results in a suboptimal knowledge transfer (higher distillation loss in validation)} as the face logits are prepared to disregard the periocular logits.
(b) The periocular classification loss $- p_k \log p_k$. Effective transfer of consistent knowledge enables CKD to \textit{generalize better (lower validation loss)} than the CE and KD baselines.
(c) As a result, the periocular embeddings of CKD form effective clusters. Here, the cluster quality is measured by the Davies-Bouldin Index (DBI), which quantifies the ratio of intra-class variation to inter-class variation.
}
\label{fig:loss_dbi}
\end{figure}

\section{Conclusion}
We presented a periocular embedding learning method enhanced by means of consistent knowledge distillation (CKD) that is robust under wild, unconstrained environments. CKD imposes temperature-based consistency between face and periocular images through prediction and feature layers, thereby effectively transferring global inter-class relationship information from face images to the periocular network and inducing robust periocular recognition. Furthermore, the design of train loss in CKD allows it to be trained end-to-end with a single stage and single hyperparameter, making it achieve state-of-the-art performance over six standard periocular benchmarks. Moreover, we have provided extensive theoretical and experimental analyses on the distillation principles of CKD and found that the novel sparsity-oriented regularizer hidden in the consistency constraint of CKD helps the face network to effectively extract the global inter-class relationship of subject identities from face images.


\bibliographystyle{elsarticle-num} 
\bibliography{refs}


%
%
%


\end{document}